\documentclass[11pt]{extarticle}
\usepackage{fullpage,subfigure,fancyhdr}
\usepackage[T1]{fontenc}
\usepackage{color}
\usepackage{verbatim}
\usepackage{times}
\usepackage{amsfonts,amsmath,amssymb,amsthm}
\usepackage{textcomp}
\usepackage{url}
\usepackage{mdwlist}
\usepackage{wrapfig,subfigure,sidecap}
\usepackage[font=small,labelfont=bf]{caption}
\usepackage{xspace}
\usepackage{algorithm}
\usepackage{algpseudocode}
\usepackage{graphicx}
\usepackage[colorlinks=true,pagebackref,citecolor=lgray,linkcolor=magenta]{hyperref}
\usepackage[sort&compress,comma,square,numbers]{natbib}
\usepackage{pict2e}
\usepackage[nottoc]{tocbibind} 
\usepackage{paralist}
\usepackage{multicol}
\usepackage{lettrine}
\usepackage{authblk}
\usepackage[linecolor=darkblue, linewidth=1.5pt, skipabove=4pt, nobreak=true]{mdframed}

\usepackage{tikz}
\usepackage{graphicx}
\usepackage{neurodata}
\usepackage{lineno}
\usepackage[final]{changes}
\usepackage{listings}
\usepackage{enumitem}
\setlist{nosep} 

\definecolor{MyGray}{rgb}{0.7,0.7,0.7}
\RequirePackage{titlesec}

\makeatletter
\titleformat*{\section}{\color{MSBlue}\bfseries\large}
\titleformat*{\subsection}{\color{MSBlue}\bfseries\large}
\titleformat*{\subsection}{\color{MSBlue}\bfseries\large}
\titleformat*{\paragraph}{\color{MSBlue}\bfseries\large}

\definecolor{MyPlum}{rgb}{0.3,0,0.3}
\definecolor{MyOrange}{rgb}{1,0.5,0}
\definecolor{deep_blue}{rgb}{0,.2,.5}
\definecolor{darkblue}{rgb}{0,.15,.5}

\providecommand{\ve}[1]{\boldsymbol{#1}}
\providecommand{\norm}[1]{\ensuremath{\left \lVert#1 \right  \rVert}}
\newcommand{\argmax}{\operatornamewithlimits{argmax}}
\newcommand{\argmin}{\operatornamewithlimits{argmin}}

\providecommand{\mc}[1]{\mathcal{#1}}
\providecommand{\mt}[1]{\widetilde{#1}}
\providecommand{\mb}[1]{\boldsymbol{#1}}

\providecommand{\mh}[1]{\hat{#1}}

\providecommand{\mtb}[1]{\widetilde{\boldsymbol{#1}}}

\newtheorem{thm}{Theorem}

\newtheorem{lem}{Lemma}
\newtheorem{defi}{Definition}

\newtheorem{proposition}{Proposition}
\newtheorem{coro}[thm]{Corollary}

\newtheorem{remark}{Remark}

\newcommand{\bla}{\begin{block}}
\newcommand{\blb}{\end{block}}

\newcommand{\defa}{\begin{defi}}
\newcommand{\defb}{\end{defi}}

\newcommand{\thma}{\begin{thm}}
\newcommand{\thmb}{\end{thm}}

\newcommand{\mata}{\begin{bmatrix}}
\newcommand{\matb}{\end{bmatrix}}

\floatname{algorithm}{Procedure}

\floatname{algorithm}{Pseudocode}

\newcommand{\Real}{\mathbb{R}}

\newcommand{\bA}{\mb{A}}
\newcommand{\bB}{\mb{B}}

\newcommand{\bI}{\mb{I}}

\newcommand{\bS}{\mb{S}}
\newcommand{\bU}{\mb{U}}
\newcommand{\bV}{\mb{V}}
\newcommand{\bW}{\mb{W}}
\newcommand{\bX}{\mb{X}}
\newcommand{\bY}{\mb{Y}}
\newcommand{\bZ}{\mb{Z}}

\newcommand{\bd}{\mb{d}}

\newcommand{\bu}{\mb{u}}

\newcommand{\bx}{\mb{x}}

\newcommand{\EE}{\mathbb{E}}           

\newcommand{\II}{\mathbb{I}}           

\newcommand{\PP}{\mathbb{P}}

\newcommand{\bth}{\ve{\theta}}

\newcommand{\bTh}{\ve{\Theta}}

\newcommand{\bSig}{\ve{\Sigma}}

\newcommand{\bmu}{\ve{\mu}}

\newcommand{\bpi}{\ve{\pi}}

\newcommand{\bdel}{\ve{\delta}}

\providecommand{\sct}[1]{{\texttt{#1}}}

\newcommand{\Svd}{\sct{SVD}}

\newcommand{\Pca}{\sct{PCA}}

\newcommand{\Lda}{\sct{LDA}}
\newcommand{\Qda}{\sct{QDA}}
\newcommand{\Cca}{\sct{CCA}}
\newcommand{\Pls}{\sct{PLS}}

\newcommand{\Lol}{\sct{LOL}}
\newcommand{\Rlol}{\sct{RLOL}}
\newcommand{\Rp}{\sct{RP}}

\newcommand{\Qoq}{\sct{QOQ}}

\newcommand{\Lfl}{\sct{LFL}}

\newcommand{\Road}{\sct{ROAD}}

\newcommand{\Lasso}{\sct{Lasso}}
\newcommand{\Rrlda}{\sct{rrLDA}}
\newcommand{\TT}{^{\ensuremath{\mathsf{T}}}}

\usetikzlibrary{calc,shapes, arrows,positioning}
\tikzstyle{node} = [rectangle, rounded corners, minimum width=1cm, minimum height=1cm]
\tikzstyle{node2}=[rectangle split,rectangle split parts=2,rounded corners]
\tikzstyle{arrow} = [thick,->,>=stealth]

\usepackage{lipsum}

\makeatletter
\renewcommand\@maketitle{%
	\noindent\parbox{1\textwidth}{
    \color{MSBlue} {\bfseries\LARGE{{\@title \par \hfill}}}
    \par
    \color{lgray}{\@author}
    }
}

\newcommand{\Xox}{\texttt{XOX}}
\begin{document}

\def\spacingset#1{\renewcommand{\baselinestretch}%
{#1}\small\normalsize} \spacingset{1}
\title{ 
Supervised Dimensionality Reduction for Big Data
}

\author{
Joshua T.~Vogelstein$^{1 *\dag}$,  
Eric W.~Bridgeford$^{1*}$,  
Minh Tang$^1$, 
Da Zheng$^1$, 
Christopher Douville$^1$,
Randal Burns$^1$, 
Mauro Maggioni$^1$
\\ $^1$ Johns Hopkins University, 
$^*$Co-First, $^\dag$Corresponding Author}
\date{}

\maketitle
\thispagestyle{empty}

\textbf{%
To solve key biomedical problems, experimentalists now routinely measure millions or billions of features (dimensions) per sample, with the hope that data science techniques will be able to build  accurate data-driven inferences.  Because  sample sizes are typically  orders of magnitude smaller than the dimensionality of these data, valid inferences require finding a low-dimensional representation that preserves the discriminating information (e.g., whether the individual suffers from a particular disease). There is a lack of interpretable supervised dimensionality reduction methods that scale to millions of dimensions with strong statistical theoretical guarantees. 
{\color{black}We introduce an approach, {XOX},} to extending principal components analysis by incorporating class-conditional moment estimates into the low-dimensional projection.  The simplest version, ``Linear Optimal Low-rank'' projection (LOL), incorporates the class-conditional means.
We prove, and substantiate with both synthetic and real data benchmarks, that LOL and its generalizations {\color{black}in the XOX framework} lead to improved data representations  for subsequent classification, while maintaining computational efficiency and scalability. 
Using multiple brain imaging datasets consisting of $>$150 million features, and several genomics datasets with  $>$500,000 features, {\color{black}LOL outperforms other scalable linear dimensionality reduction techniques in terms of accuracy, while only requiring a few minutes on a standard desktop computer.}
}

Supervised learning---the art and science of estimating statistical relationships using labeled training data---has  enabled a wide variety of basic and applied findings, ranging from discovering biomarkers in omics data \cite{Vogelstein2014a} to  recognizing objects from images \cite{Krizhevsky2012}.
A special case of supervised learning is classification, where a classifier predicts the ``class'' of a novel observation (for example, by predicting sex from an MRI scan). One of the most foundational and important approaches to classification is Fisher's Linear Discriminant Analysis (\Lda) \cite{Fisher1925a}.
\Lda~has a number of highly desirable properties for a  classifier.
First, it is based on  simple geometric reasoning: when the data are Gaussian, all the information is in the means and variances, so the optimal classifier uses both the means and the variances.
Second,  \Lda~can be applied to multiclass problems.
Third, theorems guarantee that when the sample size $n$ is large and the dimensionality $p$ is relatively small, \Lda~converges to the optimal classifier under the Gaussian assumption.
Finally,  algorithms for implementing it are highly efficient.

Modern scientific datasets, however, present challenges for classification that were not addressed in Fisher's era.
Specifically, the dimensionality of datasets is quickly ballooning.
Current raw data can consist of hundreds of millions  of features or dimensions; for example, an entire genome or connectome.  Yet, the sample sizes have not experienced a concomitant increase.
This ``large $p$, small $n$'' problem is a non-starter
for many classical statistical approaches because they were designed with a ``small $p$, large $n$'' situation in mind.
Running \Lda~when $p \ge n$  is like trying to fit a line to a point: there are infinitely many equally good fits (all lines that pass through the point), and no way to know which of them is ``best''.
Therefore, without further constraints these algorithms will overfit, meaning they will choose a classifier based on noise in the data, rather than discarding the noise in favor of the desired signal.
We also desire methods that  can adapt to the complexity of the data,  are robust to outliers, and are computationally efficient.
Several complementary strategies have been pursued to address these $p \ge n$ problems.

First, and perhaps the most widely used method, is Principal Components Analysis (\Pca) \cite{Jolliffe1986}. According to PubMed, \Pca~has been referenced over 40,000 times, and nearly 4,000 times in 2018 alone.  This is in contrast to other methods that receive much more attention in the media, such as deep learning, random forests, and sparse learning, which received $\sim$ 2,000, $\sim$ 1,200  and $\sim$ 500 hits, respectively. This suggests that \Pca~remains the most popular workhorse  for high-dimensional problems.  \Pca~ ``pre-processes'' the data by reducing its dimensionality  to those dimensions whose variance is largest in the dataset.  While highly successful, \Pca~is a wholly \emph{unsupervised} dimensionality reduction technique, meaning that \Pca~does not use the class labels while learning the low-dimensional representation, resulting in sub-optimal performance for subsequent classification.  Nonlinear manifold learning techniques generalize \Pca~\cite{Lee2007-bw}, but also typically do not incorporate class label information; moreover, they scale poorly. Deep learning provides the most recent version of nonlinear manifold learning, for example, using (supervised) autoencoders, but these methods remain poorly understood, have many parameters to tune, and typically do not provide interpretable results~\cite{Goodfellow2016-ac}. Further, deep learning tends to suffer in the wide data problem, where the number of samples is far less than the dimensionality.

The second set of strategies  regularize or penalize a supervised method, such as regularized~\Lda~\cite{Witten2009a} or canonical correlation analysis (\Cca)~\cite{Shin11}.  Such approaches can drastically overfit in the $p>n$ setting, tend to lack theoretical support in these contexts,  and have multiple ``knobs'' to tune that are computationally taxing.
Partial least squares (\Pls) is another popular method in this set that often achieves impressive empirical performance, though it lacks strong theoretical guarantees and a scalable implementation~\cite{Ter_Braak1998-cc, Brereton2014-wr}.
Sparse methods are the third common strategy to mitigate this ``curse of dimensionality''~ \cite{Tibshirani1996,Fan2012a,Hastie2015}. Unfortunately, exact solutions are computationally intractable, and approximate solutions have theoretical guarantees only under very restrictive assumptions, and are quite fragile to those assumptions~\cite{Su2015}.
Thus, there is a gap: no existing  approach  can classify multi-class {\color{black}wide} data with millions of features while  obtaining strong theoretical guarantees,  favorable and interpretable empirical performance, and a flexible, robust, and scalable implementation.

To address these issues,  {\color{black}we developed a technique for incorporating class-conditional moment estimates, \Xox, the simplest example of which is \Lol}.
\emph{The key intuition behind \Lol~is that we can jointly use the means and variances from each class  (like \Lda~and \Cca), but without requiring more dimensions than samples   (like \Pca), or restrictive sparsity assumptions.}
Using random matrix theory, we are able to prove that when the data are sampled from a  Gaussian, \Lol~finds a better low-dimensional representation than \Pca, \Lda, \Cca, and  other  linear methods. Under  relatively relaxed assumptions, this is true regardless of the dimensionality of the features, the number of samples, or the number of dimensions in which we project.
We then demonstrate the superiority of techniques derived using the \Xox~ approach {\color{black}---including (i) \Lol, (ii) a variant of {\color{black}\Xox}~which allows greater flexibility of the class-conditional covariances called \Qoq, and (iii) a robust variant of \Lol~called \Rlol---}over other methods numerically on a variety of simulated settings including several not following the theoretical assumptions. Finally, we show that on several 500 gigabyte neuroimaging datasets, 
and several multi-gigabyte genomics datasets, 
\Lol~achieves superior accuracy at lower dimensions while requiring only a few minutes of time on a single workstation.

\subsection*{Supervised Manifold Learning}

\begin{figure}
\centering 
\includegraphics[height=2in]{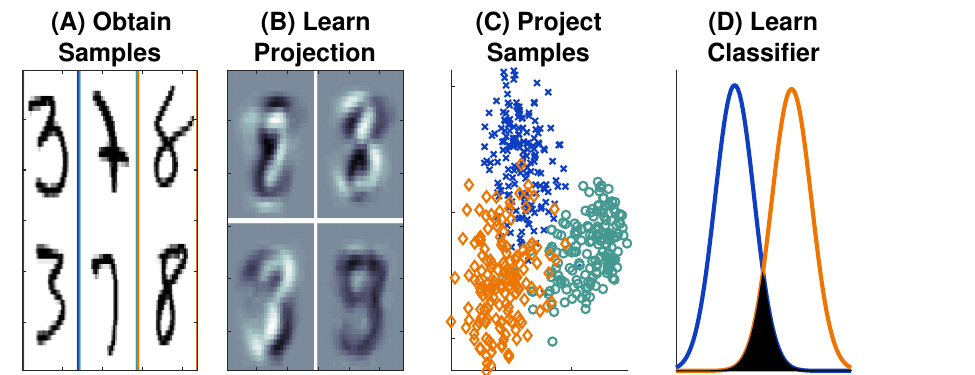} 
\caption{Schematic illustrating Linear Optimal Low-rank (\Lol)  as a supervised manifold learning technique.
\textbf{(A)} 300 training samples of the numbers 3, 7, and 8 from the MNIST dataset (100 samples per digit);  each sample is a 28 $\times$ 28 = 784 dimensional image (boundary colors are  for visualization purposes).
\textbf{(B)} The first four projection matrices learned by \Lol.  Each is a linear combination of the sample images.
\textbf{(C)} Projecting 500 new (test) samples into the top two learned dimensions;
 digits color coded as in (A). \Lol-projected data from three distinct clusters.
\textbf{(D)}  Using the low-dimensional data to learn a classifier.  The estimated  distributions for 3 and 8 of the test samples (after
projecting data into two dimensions and using \Lda~to classify) demonstrate that 3 and 8 are easily separable by linear methods after \Lol~projections (the color of the line indicates the digit).
The filled area is the estimated error rate; the goal of any classification algorithm is to minimize that area. \Lol~is performing well on this high-dimensional real data example.
}
\label{f:mnist}
\end{figure}

A general strategy for supervised manifold learning is schematized in Figure~\ref{f:mnist}, and outlined here.
Step \textbf{(A)}: Obtain or select $n$ training samples of high-dimensional data.  For concreteness, we use one of the most popular benchmark datasets, the MNIST dataset \cite{mnist}.  This dataset consists of  images of hand-written digits $0$ through $9$.  Each image is represented by a $28\times28$ matrix, which means that the observed  dimensionality of the data is $p=28^2=784$.  Because we are motivated by the $n \ll p$ scenario, we subsample the data to select $n=300$ examples of the numbers $3$, $7$, and $8$ ($100$ of each).
Step \textbf{(B)}: Learn a ``projection'' that maps the high-dimensional data to a low-dimension representation.  One can do so in a way that ignores which images correspond to which digit (the ``class labels''), as \Pca~and most manifold learning techniques do, or try to use the labels, as \Lda~and sparse methods do.
 \Lol~is a supervised linear manifold learning technique that  uses the class labels to learn projections that are linear combinations of the original data samples. 
Step \textbf{(C)}: Use the learned projections to map  high-dimensional data into the learned lower-dimensional space. This step requires having learned a projection that can be applied to new (test) data samples for which we do not know the true class labels.  Nonlinear manifold learning methods typically cannot be applied in this way (though see \cite{Bengio2004}).  \Lol, however, can project new samples in such a way as to separate the data into classes.
Step \textbf{(D)}: Using the low-dimensional representation of the data, learn a classifier.  A good classifier correctly identifies as many points as possible with the correct label.  
For these data, when  \Lda~is used on the low-dimensional data learned by \Lol, the data points are mostly linearly separable, yielding a highly accurate classifier.

\subsection*{The Geometric Intuition of \Lol}

To build intuition for situations when \Lol~performs well, and when it does not,
we consider the simplest high-dimensional classification setting.
We observe $n$ samples $(\bx_i,y_i)$, where $\bx_i$ are $p$ dimensional feature vectors, and $y_i$ is the binary class label, that is, $y_i$ is either $0$ or $1$.
We assume that both classes are distributed according to a multivariate Gaussian distribution,  the two classes have the same identity covariance matrix  (all features are uncorrelated with unity variance), and data from either class is equally likely, so that the only difference between the classes is their means.
In this scenario, the optimal low-dimensional projection is analytically available: it is the dot product of the difference of means and the inverse covariance matrix, commonly referred to as Fisher's Linear Discriminant Analysis (\Lda) \cite{Bickel2004a} (see Appendix~\ref{sec:background} for derivation).
When the distribution of the data is unavailable, as in all real data problems, machine learning methods can be used to estimate the parameters.
Unfortunately, when $n<p$, the estimated covariance matrix will not be invertible (because the solution to the underlying mathematical problem is under specified), so some other approach is required.
As mentioned above, \Pca~is commonly used to learn a low-dimensional representation.
\Pca~uses the pooled sample mean
and the pooled sample covariance matrix.
The \Pca~projection is composed of the top $d$ eigenvectors of the pooled sample covariance matrix, after subtracting the pooled mean (thereby completely ignoring the class labels).

In contrast, \Lol~uses the class-conditional means and class-centered covariance.
This approach is motivated  by Fisher's \Lda, which uses the same two terms, and should therefore improve performance over \Pca.
More specifically, for a two-class problem, \Lol~is constructed as follows:
\begin{enumerate}
    \item Compute the {sample mean of each class}.
    \item Estimate the difference between  means.
    \item Compute the class-centered covariance matrix, that is, compute the covariance matrix after subtracting the class mean from each point.
    \item Compute the eigenvectors of this class-conditionally centered covariance.
    \item Concatenate the difference of the means with the top $d-1$ eigenvectors of class-centered covariance.
\end{enumerate}
Note that the sample class-centered covariance matrix estimates the population covariance,
{color{red}whereas the sample pooled covariance matrix is distorted by the difference of the class means. Further, as discussed in Appendix \ref{sec:main}, the class-centered covariance matrix is equivalent to ``Reduced Rank \Lda'' \cite{Hastie1996}  ({\color{black}\Rrlda}~hereafter, which is simply \Lda~but truncating the covariance matrix)}.
For the theoretical background on \Lda~ and {\color{black}\Rrlda},  a formal definition of \Lol, and detailed description of the simulation settings that follow, see Appendices \ref{sec:background}, \ref{sec:LOL}, and \ref{sec:simulations}, respectively.
Figure \ref{f:cigars} shows three different examples of 100 data points sampled from a 1,000 dimensional Gaussian  to geometrically illustrate the intuition that motivated \Lol.  In each case, all dimensions are uncorrelated with one another,  and all classes are equally likely with the same covariance; the only difference between the classes are their means.

\begin{figure}[h!]
\centering
\includegraphics[width=0.7\linewidth]{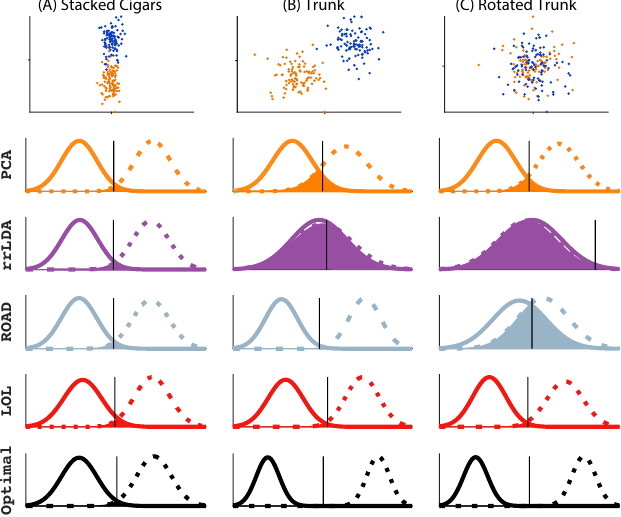}
\caption{
\Lol~achieves near-optimal performance for three different multivariate Gaussian distributions,
each with $100$ samples in $1000$ dimensions.
For each approach, we project into the top 3 dimensions, and then use \Lda~to classify 10,000 new samples.
The six rows show (from top to bottom):
\emph{Row 1}: A scatter plot of the first two dimensions of the sampled points, with class 0 and 1 as orange and blue dots, respectively.
The next rows each show the estimated posterior for class 0 and class 1, in solid and dashed lines, respectively. The overlap of the distributions---which quantifies the magnitude of the error---is filled.  The black vertical line shows the estimated threshold for each method.
The techniques include:
\Pca;
reduced rank \Lda ({\color{black}\Rrlda}), a method that projects onto the top $d$ eigenvectors of  sample class-conditional covariance;
\Road, a sparse method designed specifically for this model;
\Lol, our  proposed method; and
the Bayes optimal classifier.
\textbf{(A) Stacked Cigars} The mean difference vector is aligned with the direction of maximal variance, and is mostly concentrated in a single dimension, making it ideal for  \Pca, {\color{black}\Rrlda}, and sparse methods.
In this setting, the results are  similar for all methods, and essentially optimal.
\textbf{(B) Trunk} The mean difference vector is orthogonal to the direction of maximal variance; \Pca~performs worse and {\color{black}\Rrlda}~is at chance, but sparse methods and \Lol~can still recover the correct dimensions, achieving nearly optimal performance.
\textbf{(C) Rotated Trunk} Same as (B), but the data are rotated; in this case, only \Lol~performs well.
Note that \Lol~is closest to Bayes optimal in all three settings.
}
\label{f:cigars}
\end{figure}

Figure \ref{f:cigars}{\color{magenta}A} shows   ``stacked cigars'',  in which the difference between the means and  the direction of maximum variance are  large  and aligned with one another.
This is an idealized setting for \Pca, because \Pca~finds the direction of maximal variance, which happens to correspond to the direction of maximal separation of the classes.
{\color{black}\Rrlda} performs well here too, for the same reason that \Pca~does.
Because all dimensions are uncorrelated, and one dimension contains most of the information discriminating between the two classes, this is also an ideal scenario for sparse methods.
Indeed,  \Road, a sparse classifier designed for precisely this scenario,  does an excellent job finding the most useful dimensions \cite{Fan2012a}.
\Lol, using both the difference of means and the directions of maximal variance, also does well.
To calibrate all of these methods, we also show the performance of the optimal classifier.

Figure \ref{f:cigars}{\color{magenta}B} shows an example that is worse for  \Pca.
In particular, the variance is getting larger for subsequent dimensions,
while the magnitude of the difference between the means is decreasing with dimension.
Because \Pca~operates on the pooled sample covariance matrix, the dimensions with the maximum difference are included in the estimate, and therefore, \Pca~finds some of them, while also finding some of the dimensions of maximum variance.  The result is that \Pca~performs fairly well in this setting.
{\color{black}\Rrlda}, however, by virtue of subtracting out the difference of the means, is now completely at chance performance.
\Road~is not hampered by this problem; it is also able to find the directions of maximal discrimination, rather than those of maximal variance.
Again, \Lol, by using both the means and the covariance, does extremely well.

Figure \ref{f:cigars}{\color{magenta}C}  is exactly the same as Figure \ref{f:cigars}{\color{magenta}B}, except the data have been randomly rotated in all 1000 dimensions.  This means that none of the original features have much information, but rather, linear combinations of them do.
This is evidenced by observing the scatter plot, which shows that the first two dimensions  fail to disambiguate the two classes.
\Pca~performs even worse in this scenario than in the previous one.
{\color{black}\Rrlda}~is rotationally invariant (see Appendix~\ref{app:rot}  for details),  so still performs at chance levels.
Because there is no small number of features that separate the data well,  \Road~fails.
 \Lol~performs  as well here as it does in the other examples.


\subsection*{When is \Lol~Better than \Pca~and Other Supervised Linear Methods?}

We desire theoretical confirmation of the above numerical results.
To do so, we investigate when \Lol~is ``better'' than other linear dimensionality reduction techniques.
In the context of supervised dimensionality reduction or manifold learning, the goal is to obtain low dimensional representation that maximally separates the two classes, making subsequent classification easier.  Chernoff information quantifies the dissimilarity between two distributions. Therefore, we can compute the Chernoff information between distribution of the two classes after embedding to evaluate the quality of a given embedding strategy.
As it turns out, Chernoff information  is the exponential convergence rate for the Bayes error \cite{chernoff_1952}, and therefore, the tightest possible theoretical bound.   The use of Chernoff information to theoretically evaluate the performance of an embedding strategy is novel, to our knowledge, and leads to the following main result:

\makebox[\textwidth][c]{%
    \begin{minipage}[c]{0.9\textwidth}
\textbf{Main Theoretical Result}
\emph{\Lol~is always better than or equal to {\color{black}\Rrlda}~under the Gaussian model when $p \geq n$, and better than or equal to \Pca~(and  many other linear projection methods) with additional (relatively weak) conditions.  This is true for all possible observed dimensionalities of the data, and the number of dimensions into which we project, for sufficiently large sample sizes. Moreover, under relatively weak assumptions, these conditions almost certainly hold as the number of dimensions increases.}
\end{minipage}
}

Formal statements of the theorems and proofs required to substantiate the above result are provided in Appendix \ref{sec:main}.  The condition for \Lol~to be better than \Pca~is essentially that the $d^{th}$ eigenvector of the pooled sample covariance matrix has less information about classification than the difference of the means vector.
The implication of the above theorem is that it is better to incorporate the mean difference vector into the projection matrix, rather than ignoring it, under basically the same assumptions that motivate \Pca.  The {degree} of improvement is a function of 
the  dimensionality of the feature set $p$, 
the number of samples $n$, 
the projection dimension $d$, 
and the parameters, but the {existence} of an improvement---or at least no worse performance---is independent of those factors.

\subsection*{Flexibility and Accuracy of {\color{black}\Xox}~Framework}

We empirically investigate the flexibility and accuracy of \Xox~using simulations that extend beyond the theoretical claims.
For three different scenarios, we sample
$100$ training samples each with  $100$ features; therefore, Fisher's \Lda~cannot solve the problem (because there are infinitely many ways to overfit).
We consider a number of different methods, including 
\Pca, {\color{black}\Rrlda}, \Pls, \Road, Random Projections (\Rp), and \Cca~to project the data onto a low dimensional space.
After projecting the data, we train either \Lda~(for the first two scenarios) or  Quadratic Discriminant Analysis (\Qda, for the third scenario), which generalizes \Lda~by allowing each class to have its own covariance matrix  \cite{Hastie2004}.
For each scenario, we evaluate the misclassification rate on held-out data.

\begin{figure}[h!]
\centering
\includegraphics[width=1\linewidth]{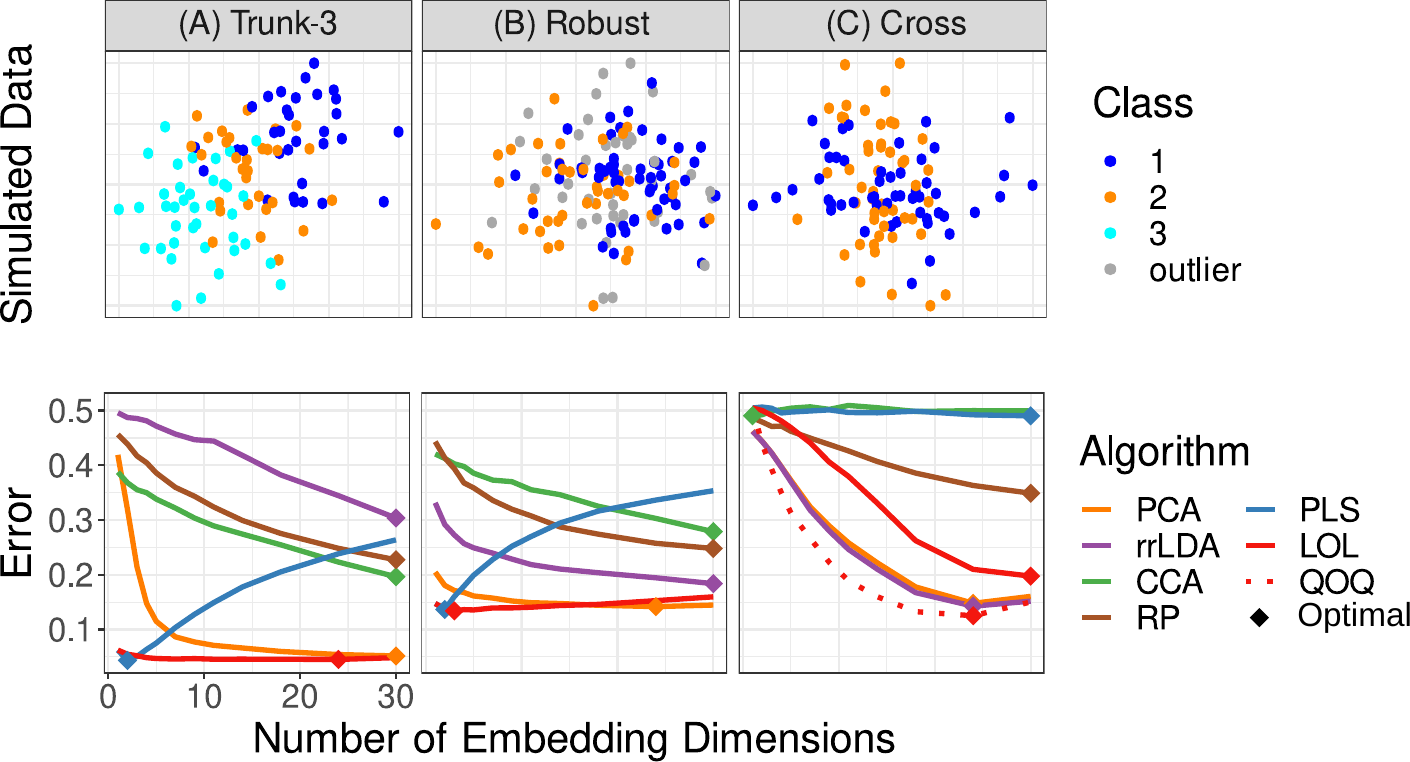}
\caption{
Three simulations demonstrating the flexibility and accuracy of \Xox~in settings beyond  current theorical claims.
For all cases, training sample size and dimensionality were both $100$. {\color{black}The top row} depicts the values of the sampled points for two of the 100 dimensions to illustrate the classification task. {\color{black}The bottom row} misclassification rate as a function of the number of projected dimensions, for several different embedding approaches. 
Classification is performed on the embedded data using the \Lda~classifier for (A) and (B), and using \Qda~for (C).
The simulation settings are:
\textbf{(A) Trunk-3} A variation of Figure {\ref{f:cigars}\color{magenta}{(B)}} in which 3 classes are present.
\textbf{(B) Robust} Outliers  are prominent in the sample while estimating the projection matrix. {\color{black}\Lol\ is robust to the outliers due to the robust estimate of the first moment.}
\textbf{(C) Cross} The two classes have the same mean but orthogonal covariances. Points are classified using the \Qda~classifier after projection. \Qoq, a variant of \Lol~where each class' covariance is incorporated into the projection matrix, outperforms other methods, as expected.
In essentially all  cases and dimensions, \Lol, or the appropriate generalization thereof, outperforms other approaches.
}
\label{f:properties}
\end{figure}
Figure~\ref{f:properties} shows a two-dimensional scatterplot (left) and misclassification rate versus dimensionality (right) for each simulation.  {\color{black} The top $C-1$ embedding dimensions for \Lol~ correspond to the performance after projection onto the class-conditional means, and \Rrlda~ corresponds to the performance of projection onto the class-conditional covariance matrix.}
Figure~\ref{f:properties}{\color{magenta}A} shows a three class generalization of the  Trunk example from Figure \ref{f:cigars}{\color{magenta}B}.
\Lol~can trivially be extended to more than two classes (see Section \ref{sec:LOL} for details), unlike \Road~which only operates in a two-class setting.
Figure~\ref{f:properties}{\color{magenta}B} shows a two-class example with many outliers, as is typical in modern biomedical datasets.  {\color{black}A variant of \Lol, ``Robust \Lol'' (\Rlol),
replaces the standard estimators of the mean and covariance with robust variants (the class-conditional medians and robust covariance \cite{cran2020Oct} respectively), thereby dramatically improving performance over \Lol~(and other techniques) in noisy settings. 
Hereafter, \Lol\ will refer to the version of \Lol\ with a robust estimate of the first moment, and a truncated estimate of the second moment, as a robust first moment tends to make little difference when a robust estimate was not necessary, and improved performance when a robust estimate was warranted. We do not use a robust estimate of the second moment, as typical robust estimates of the second moment available in standard numerical packages require $d < n$. which is unsuitable for wide data.
}
Figure~\ref{f:properties}{\color{magenta}C} shows an example that does not have an effective linear discriminant boundary because the two classes have orthogonal  covariances.
Another variant of \Lol, Quadratic Optimal \Qda~(\Qoq),
computes the eigenvectors separately for each class, concatenates them (sorting them according to their singular values), and then classifies with \Qda~instead of \Lol. 
For all three scenarios, either \Lol---or its extended variants \Rlol~and \Qoq---achieves a misclassification rate  comparable to or lower  than other methods, for all dimensions. 
%
{\color{black}These three results demonstrate how straightforward generalizations of \Lol~under the \Xox~framework which incorporate alternate or robust moment estimates} can dramatically improve performance over other projection methods. This is in marked contrast to other approaches, for which such flexibility is either not available, or otherwise problematic.

\subsection*{\Xox~is Computationally Efficient and Scalable}
\label{sec:speed}

When the dimensionality  is large (e.g., millions or billions), the main bottleneck  is sometimes merely the ability to run anything on the data, rather than its predictive accuracy.
We evaluate the computational efficiency and scalability of \Lol~in the simplest setting: two classes of spherically symmetric Gaussians (see Appendix~\ref{sec:simulations} for details) with dimensionality varying from 2 million to 128 million, and 1000 samples per class.
Because \Lol~admits a closed form solution, it can leverage highly optimized linear algebra routines rather than the costly iterative programming techniques currently required for sparse or dictionary learning type problems~\cite{Mairal2009}.
To demonstrate these computational capabilities, we built \sct{FlashLOL}, an efficient scalable \Lol~implementation with R bindings, to complement the R package  used for the above figures.

Four properties of \Lol~enable its scalable implementation. 
First, \Lol~is linear in both sample size and dimensionality (Figure \ref{f:speed}{\color{magenta}A}, solid red line).
Second, \Lol~is easily parallelizable using recent developments in ``semi-external memory'' \cite{FlashGraph, FlashMatrix, FlashEigen} 
(Figure \ref{f:speed}{\color{magenta}A}, dashed red line demonstrates that \Lol~is also linear in the number of cores).
Also note that \Lol~does not incur any meaningful additional computational cost over \Pca~(orange dashed line).
Third,  \Lol~can use  randomized approximate algorithms for eigendecompositions to further accelerate its performance~\cite{Candes2006b, Hastie2006}  (Figure \ref{f:speed}{\color{magenta}A}, orange lines).
\sct{FlashLFL}, short for Flash Low-rank \emph{Fast} Linear embedding, achieves an order of magnitude improvement in speed when using very sparse random projections instead of the eigenvectors.
Fourth, hyper-parameter selection for \Lol~is nested, meaning that once estimating the $d$-dimensional projection, every lower dimensional projection is automatically available.  This is in contrast to tuning the weight of a penalty term, which leads to a new optimization problem for each different parameter values.
Thus, the computational complexity of \Lol~is $\mc{O}(npd/Tc)$, where $n$ is sample size, $p$ is the dimension of the data, $d$ is the dimension of the projection, $T$ is the number of threads, and $c$ is the sparsity of the projection.

\begin{figure}[h!]
\centering
\includegraphics[width=0.8\linewidth]{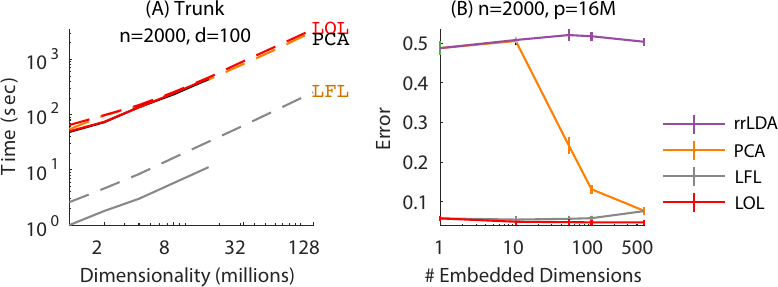}
\caption{
Computational efficiency and scalability of \Lol~using $n=2000$ samples from spherically symmetric  Gaussian data  (see Appendix~\ref{sec:simulations} for details).   \textbf{(A)} \Lol~exhibits optimal (linear) scale up, requiring only 46 minutes to find the projection on a 500 gigabyte dataset, and only 3 minutes using \Lfl~(dashed lines show semi-external memory performance). \textbf{(B)} Error for \Lfl~is the same as \Lol~in this setting, and both are significantly better than \Pca~and {\color{black}\Rrlda}~for all choices of projection dimension, {\color{black}regardless of whether a randomized approach is used to compute the projection dimensions.}
}
\label{f:speed}
\end{figure}

Finally, note that this simulation setting  is ideal for \Pca~and {\color{black}\Rrlda}, because the first principal component includes the mean difference vector.
Nonetheless, both \Lol~and \Lfl~achieve near optimal accuracy, whereas {\color{black}\Rrlda}~is at chance, and \Pca~requires 500 dimensions to even approach the same accuracy that \Lol~achieves with only one dimension. {\color{black}While \Pca~would also benefit efficiency wise from a randomized approach, we emphasize that \Lfl~maintains the high performance of \Lol~in comparison to \Pca~despite the randomization technique, with the benefit of greater computational efficiency compared to \Lol.}  

\subsection*{Real Data Benchmarks and Applications}

Real data often break the theoretical assumptions in more varied ways than the above simulations, and can provide a complementary perspective on the performance properties of different algorithms. We describe two sets of problems, one from brain imaging, and the other from genomics. In both cases we consider a classification problem.
To classify participants, researchers typically employ substantiative pre-processing pipelines \cite{discr} to reduce the dimensionality of the data. Unfortunately, as debates persist about the validity of pre-processing approaches, there is no defacto ``standard'' for the optimal strategies to pre-process the data. Traditional approaches typically include a deep processing chain, with many steps of parametric modeling and downsampling~\cite{mrcap, migraine, sic}.
We therefore investigate the possibility of directly classifying on the nearly raw, high-dimensional data.

The Consortium for Reliability and Reproducibility (CoRR) \cite{corr} has generated anatomical and diffusion magnetic resonance imaging  scans from $n > 800$ participants from $5$ processing sites, each featuring participant-specific annotations for the sex of each individual. At the native resolution, each brain volume is over $150$ million dimensions, and each dataset consists of between $42$ ($60$ GB of data) and $>400$ samples ($600$ GB of data). 

\begin{figure}[h!]
\centering
\includegraphics[width=\linewidth]{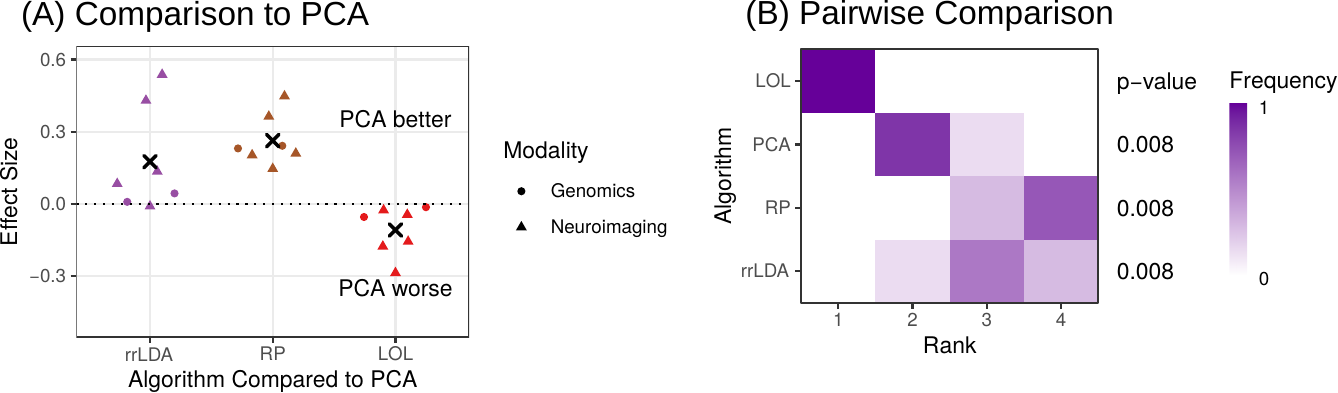}
\caption{Comparing various dimensionality reduction algorithms  on two real datasets: neuroimaging and genomics.
\textbf{(A)} Beeswarm plots show the classification performance of each technique with respect to \Pca~ at the optimal number of embedding dimensions, the number of embedding dimensions with the lowest misclassification rate. Performance is measured by the effect size, defined as $\kappa(\Lda \circ \Pca) - \kappa(\Lda \circ \textrm{embed})$, where $\kappa$ is Cohen's Kappa, and $\textrm{embed}$ is one of the embedding techniques compared to \Pca. Each point indicates the performance of \Pca~ relative the other technique on a single dataset, and the sample size-weighted average effect is indicated by the black ``x.'' \Lol~ always outperforms \Pca~and all other techniques.  \textbf{(B)} Frequency histograms of the relative ranks of each of the embedding techniques on each dataset after classification, where a $1$ indicates the best relative classification performance and a $4$ indicates the worst relative classification performance, after embedding with the technique indicated. Projecting first with \Lol~ provides a significant improvement over competing strategies (Wilcoxon signed-rank test, $n=7$, $p$-value=$.008$) on all benchmark problems.
}
\label{f:neuroimaging}
\end{figure}
\clearpage

We then also consider a large genomics dataset \cite{Douville2020Mar} consisting of  $340$ individuals:  $144$ patients with non-metastatic cancer and $196$ healthy controls, of which $198$ are male and $142$ are female. Samples are aligned to $>$750,000 amplicons distributed throughout the genome to investigate the presence of aneuploidy (abnormal chromosomal counts) in samples from cancer patients (see Appendix \ref{app:neuroimaging} for details). The raw amplicon counts are then  used with no further pre-processing.
We have two tasks of interest: classification on the basis of either sex or age. 
%

For each of the above described problems, we first compute an embedding matrix to project the training data using \Lol, \Pca, {\color{black}\Rrlda}, and \Rp, and then train \Lda~to classify the resulting low-dimensional representations. The \textit{held-out} set is then projected and classified using the embedding matrix and trained classifier respectively, and the average cross-validated error is computed over all folds of the data. For each problem, the optimal dimensionality for each strategy is selected to be the number of embedding dimensions with the lowest average cross-validated error. We compute Cohen's Kappa $\kappa$ to compare performance across methods because it normalizes the performance of the classification strategy between zero (the classifier is equivalent to the \texttt{random chance} classifier) and one (the classifier performs perfectly).
Finally, for each projection technique, we measure the effect size for each strategy as the difference $\kappa(\Pca) - \kappa(\textrm{embed})$. See Appendix \ref{app:neuroimaging} for a table detailing the datasets employed.

Our  \texttt{Flash}\Lol~implementations  are the only algorithms  that could  successfully run on these data with a single core on a standard desktop computer. In Figure \ref{f:neuroimaging}(A), \Lol~is the only technique to outperform \Pca~ on all problems. Figure \ref{f:neuroimaging}(B) shows the relative ranks of the average cross-validated misclassification rates for the \Lda~classifier on each dataset after projection with the specified embedding technique. For all problems, \Lol~is the technique with the lowest average cross-validated misclassification rate. Further, \Lol~performs significantly better than all other techniques (Wilcoxon signed-rank statistic, all $p$-values$=0.008$). The average misclassification rate achieved at the optimal number of embedding dimensions via \Lol~is between $5\%$ and $15\%$ across all datasets, which is the same performance we and others obtain using  extensively processed and downsampled data that is typically required on similar datasets \cite{Vogelstein2013,Duarte-Carvajalino2011}. \Lol~therefore enables researchers to side-step hotly debated pre-processing issues by hardly pre-processing at all,  and instead simply applying \Lol~to the data in its native dimensionality.

\section*{Discussion}

{\color{black}We have introduced a very simple methodology to improve performance on supervised learning problems with  wide data (that is, big data where dimensionality is at least as large as sample size) by using class-conditional moments to estimate a low rank projection under a generalized framework, \Xox. In particular, \Lol~uses both the difference of the means and the class-centered covariance matrices}, which enables it to outperform \Pca, as well as existing supervised linear classification schemes, in a wide variety of scenarios without incurring any meaningful additional computational cost. Straightforward generalizations enable robust and nonlinear variants by using robust estimators and/or class specific covariance estimators.
Our open source implementation optimally scales to terabyte datasets. Moreover, the intuition can be extended for both hypothesis testing and regression  (see Appendix \ref{sec:other} for additional numerical examples in these settings).

Two commonly applied approaches in these settings are partial least squares (\Pls) and canonical correlation analysis (\Cca). \Cca~is equivalent to {\color{black}\Rrlda}~whenever $p<n$, which is not of interest here.
When $p\ge n$, \Cca~and {\color{black}\Rrlda}~are not equivalent; however, in such settings, \Cca~exhibits the ``maximal data piling problem''~\cite{maximum1}  (see Appendix~\ref{sec:cca} for details).  Specifically, all the points in each class are projected onto the exact same point. This results in severe overfitting of the data, yielding poor empirical performance in essentially all settings we considered here (the first dimension of \Cca~is typically worse even than the difference of the means).  While \Pls~does not exhibit these problems, it lacks strong theoretical guarantees and simple geometric intuition.  In contrast to {\color{black}\Xox}, neither \Cca~nor \Pls~enable straightforward generalizations, such as when there are outliers or the discriminant boundary is quadratic (see Figure~\ref{f:properties}). Further, across all simulations, {\color{black}\Xox}~outperforms both of these approaches, sometimes quite dramatically (for example, {\color{black}\Xox}~outperforms \Cca~on over all of the simulations considered). Finally, no scalable nor parallelized implementations are readily available for these methods (see Figure~\ref{f:speed}). One could use stochastic gradient descent with penalties to solve these other optimization problems, but they would still need to tune the penalty parameter which would be quite computationally costly. Neither \Pls~nor \Cca~could be successfully run on the massive neuroimaging dataset nor the amplicon-level genomics dataset using readily-available tools.

%
Many previous investigations have addressed similar challenges.
The celebrated Fisherfaces paper was the first to compose Fisher's \Lda~with \Pca~(equivalent to \Pca~in this manuscript) \cite{Belhumeur1997a}.  The authors showed via a sequence of numerical experiments the utility of projecting the data using \Pca~prior to classifying with \Lda.  
We extend this work by adding a supervised component to the initial projection.  Moreover, we provide the geometric intuition for why and when incorporating supervision is advantageous, with numerous examples demonstrating its superiority, and theoretical guarantees formalizing when \Lol~outperforms \Pca.
The ``sufficient dimensionality reduction'' literature has  similar insights, but a different construction that typically requires the dimensionality to be smaller than the sample size \cite{Li1991a, Tishby1999a, Globerson2003a, Cook2005a,Fukumizu2004a} (although see \cite{Cook2013} for some promising work). More recently, communication-inspired classification approaches have yielded theoretical bounds on linear and affine classification performance \cite{Nokleby2015}; they do not, however, explicitly compare different projections, and the bounds we provide are more general and tighter.
Moreover, none of the above strategies have implementations that scale to millions or billions of features.
Recent big data packages are designed for millions or billions of  samples~\cite{Agarwal2014,abadi2016tensorflow}.  In biomedical sciences, however, it is far more common to have tens or hundreds of samples, and millions or billions of features (e.g., genomics or connectomics).

Most manifold learning methods, while exhibiting both strong theoretical \cite{Eckart1936a,deSilva2003, Allard2012} and empirical performance, are typically fully unsupervised.  Thus, in classification problems, they discover a low-dimensional representation of the data, ignoring the labels.  This approach can be highly problematic when the discriminant dimensions and the directions of maximal variance in the learned manifold are not aligned (see Figure \ref{f:mnist} for some examples).
Moreover, nonlinear manifold learning techniques tend to learn a mapping from the original samples to a low-dimensional space, but do not learn a projection, meaning that new samples cannot easily be mapped onto the low-dimensional space, a requirement for supervised learning. {\color{black}Deep learning methods~\cite{Goodfellow2016-ac} can easily be supervised, but they tend to require huge sample sizes, lack theoretical guarantees, or are opaque ``black-boxes'' that are insufficient for many biomedical applications.
This yields a dearth of ``out of the box'' supervised scalable dimensionality reduction techniques with strong theoretical guarantees with respect to classification performance bounds designed for wide datasets.}
 Random forests circumvent many of these problems, but implementations that operate on millions of dimensions do not exist~\cite{Tomita2017-mv}{\color{black}, and often produce embeddings that perform no better than \Pca~ on wide datasets (Figure \ref{f:neuroimaging}).} 
%

{\color{black}Recently discussed strategies have identified techniques for recovering the loading vectors for \Pca\ from the weights of a single-layer auto-encoder \cite{Plaut2018Apr}. Similarly, we believe there may be promise in devising a harmony between \Lol\ and supervised auto-encoder strategies to identify the spaces spanned by the first and second moments in unison. Successive techniques, such as unsupervised regularization of supervised autoencoders \cite{BibEntry2020Nov} and supervised dictionary learners \cite{BibEntry2013Sep}, may show promise for constructive development of the projection matrix through optimization, rather than estimation, techniques. Unfortunately, these approaches lack standard numerical packages for direct comparison, evaluation, and implementation. Future work may seek to highlight the similarities, or differences, possible through such techniques.}

Other approaches formulate an optimization problem, such as projection pursuit \cite{Huber1985a} and empirical risk minimization \cite{Belkin2006a}.  These methods are limited because they are prone to fall into local minima,  require costly iterative algorithms, lack any theoretical guarantees on classification accuracy \cite{Belkin2006a}.
Feature selection strategies,  such as higher criticism thresholding \cite{Donoho2008a} effectively filter the dimensions, possibly prior to performing \Pca~on the remaining features \cite{Bair2006}. These approaches could be combined with \Lol~in ultrahigh-dimensional problems.
Similarly, another  recently proposed supervised \Pca~variant builds on the elegant Hilbert-Schmidt independence criterion \cite{Gretton05} to learn an embedding \cite{Barshan2011}.  Our theory demonstrates that under the Gaussian model, composing this linear projection with the difference of the means will improve subsequent performance under general settings, implying that this will be a fertile avenue to pursue. A natural extension to this work would therefore be to estimate a Gaussian mixture model per class, rather than simply a Gaussian per class, and project onto the subspace spanned by the collection of all Gaussians.

%
{\color{black}In conclusion, the key \Xox~idea, appending class-conditional moment estimates} to convert unsupervised manifold learning to supervised manifold learning, has many potential applications and extensions.  {\color{black}We have presented the first few, including \Lol, \Qoq, and \Rlol, which demonstrated the flexibility of \Xox~under both theoretical and benchmark settings}. Incorporating additional nonlinearities {\color{black}via higher order moments}, kernel methods \cite{Mika1999a}, ensemble methods \cite{Cannings2015} such as random forests \cite{Breiman2001a}, and multiscale methods \cite{Allard2012}
are all of immediate interest.

{\color{black}
\paragraph{Data Availability} 
Data used within this manuscript are available from \url{https://neurodata.io/lol/} and \url{https://neurodata.io//mri}.}

{\color{black}
\paragraph{Code Availability} 
MATLAB,  R, and Python code for the experiments performed in this manuscript and a docker container for \texttt{FlashLOL} are available from \url{https://neurodata.io/lol/}, and an R package is available on the Comprehensive R Archive Network (CRAN)~\cite{Bridgeford2018-fq}.}


\clearpage
\pagestyle{empty}
\bibliographystyle{plainnat}
\bibliography{biblol}

\section*{Acknowledgements}

The authors are grateful for the support  by the XDATA program of the Defense Advanced Research Projects Agency (DARPA) administered through Air
Force Research Laboratory contract FA8750-12-2-0303;
DARPA GRAPHS contract N66001-14-1-4028; and
DARPA SIMPLEX program through SPAWAR contract N66001-15-C-4041, and DARPA Lifelong Learning Machines program through contract FA8650-18-2-7834.

\clearpage
\appendix

\section{Theoretical Background}
\label{sec:background}

\subsection{The Classification Problem}

Let $(\bX,Y)$ be a pair of random variables, jointly sampled from $F :=F_{\bX,Y}=F_{\bX|Y}F_{Y}$, with density denoted $f_{\bX,Y}$.
Let $\bX$ be a multivariate vector-valued random variable, such that its realizations live in p dimensional Euclidean space, $\bx \in \Real^p$.  Let $Y$ be a categorical random variable, whose realizations are discrete, {\color{black}$y \in \{0,\ldots C-1\}$}.  The goal of a classification problem is to find a function $g(\bx)$ such that its output tends to be the true class label $y$:
\begin{align*} 
g^*(\bx) := \argmax_{g \in \mc{G}} \PP[g(\bx) = y].
\end{align*}
When the joint distribution of the data is known, then the Bayes optimal solution is:
\begin{align}  \label{eq:R}
g^*(\bx) := \argmax_y f_{y,\bx} = \argmax_y f_{\bx|y}f_y =\argmax_y \{\log f_{\bx|y} + \log f_y \}
\end{align}
Denote expected misclassification rate of classifier $g$ for a given joint distribution $F$,
\begin{align*}
L^F_g := \EE[g(\bx) \neq y] := \int \PP[g(\bx) \neq y] f_{\bx,y} d\bx dy,
\end{align*}
where $\EE$ is the expectation, which in this case, is with respect to $F_{XY}$.
For brevity, we often simply write $L_g$, and we define $L_* := L_{g^*}$.

\subsection{Linear Discriminant Analysis (\Lda)}

Linear Discriminant Analysis (\Lda) is an approach to classification that uses a linear function of the first two moments of the distribution of the data.  More specifically, let $\bmu_j=\EE[F_{X|Y=j}]$ denote the class conditional mean, and let $\bSig=\EE[F_{X}^2]$ denote the joint covariance matrix, and the class priors are $\pi_j=\PP[Y=j]$.  {\color{black} Using this notation, we can define the \Lda~classifier:
\begin{align*}
g_{\Lda}(\bx)&:=\argmin_y \left[\frac{1}{2} (\bx-\bmu_y)\TT \bSig^{-1}(\bx-\bmu_y) - \log \pi_y\right],
\end{align*}
Let $L_{\Lda}^F$ be the expected misclassification rate of the above classifier for distribution $F$.
{\color{black}Assuming equal class prior and centered means, 
re-arranging a bit, we obtain}
\begin{align*}
g_{\Lda}(\bx) :=  \argmin_y \bx\TT \bSig^{-1} \bmu_y.
\end{align*}
In words, the  \Lda~classifier chooses the class that maximizes the magnitude of the projection of an input vector $\bx$ onto $\bSig^{-1} \bmu_y$.
When there are only two classes, letting $\bdel=\bmu_0-\bmu_1$, the above further simplifies to
\begin{align*}
g_{2\Lda}(\bx) :=  \II\{ \bx\TT \bSig^{-1} \bdel > 0 \}.
\end{align*}
Note that the equal class prior and centered means assumptions merely changes the threshold constant from $0$ to some other constant.}

\subsection{\Lda~Model}

A statistical model is  a family of distributions indexed by a parameter $\bth \in \bTh$, $\mc{F}_{\bth}=\{F_{\bth} : \bth \in \bTh \}$.
Consider the special case of the above where $F_{\bX|Y=y}$ is a multivariate Gaussian distribution,
$\mc{N}(\bmu_y,\bSig)$, where each class has its own mean, but all classes have the same covariance.
We refer to this model as the \Lda~model.
Let $\bth=(\bpi,\bmu,\bSig)$, and let $\bTh_{C-\Lda}=( \triangle_C, \Real^{p \times C},\Real_{\succ 0}^{p \times p})$, where $\bmu=(\bmu_1,\ldots, \bmu_C)$, $\triangle_C$ is the $C$ dimensional simplex, that is $\triangle_C = \{ \bx : x_i \geq 0 \forall i, \sum_i x_i = 1\}$, and $\Real_{\succ 0}^{p \times p}$ is the set of positive definite  $p \times p$ matrices. Denote
$\mc{F}_{\Lda}=\{F_{\bth} : \bth \in \bTh_{\Lda}\}$.
The following lemma is well known \cite{Carson2012Jun}:
\begin{lem}
$L_{\Lda}^F=L_*^F$ for any $F \in \mc{F}_{\Lda}$.
\end{lem}


\section{Formal Definition of \Lol~and Related Projection Based Classifiers}
\label{sec:LOL}

Let $\bA \in \Real^{d \times p}$ be a ``projection matrix'', that is, a
matrix that projects $p$-dimensional data into a $d$-dimensional subspace.
The question that motivated this work is: what is the best projection matrix that we can estimate, to use to ``pre-process'' the data prior to classifying the data?
Projecting the data $\bx$ onto a low-dimensional subspace, and then classifying via \Lda~in that subspace is equivalent to redefining the parameters in the low-dimensional subspace,
$\bSig_A=\bA \bSig \bA\TT \in \Real^{d \times d}$ and $\bdel_A = \bA \bdel \in \Real^d$, and then using $g_{\Lda}$ in the low-dimensional space.  When $C=2$, $\pi_0=\pi_1$, and $(\mu_0+\mu_1)/2=\mb{0}$, this amounts to:
\begin{align} \label{eq:g_A}
g^d_A(x) := \II \{ (\bA \bx)\TT \bSig^{-1}_A \bdel_A > 0\}, \text{ where } \bA \in \Real^{d \times p}.
\end{align}
Let ${\color{black}L^d_A :=\int \PP[g_A^d(\bx)=y] f_{\bx,y} d\bx dy}$.
Our goal therefore is to be able to choose $A$ for a given parameter setting $\bth=(\bpi, \bdel,\bSig)$, such that $L_A$ is as small as possible (note that $L_A$ will never be smaller than $L_*$).

In the naive case where $(\bSig_A, \bdel_A)$ are known, {\color{black}we seek to solve the following linear optimization problem:
\begin{equation}
\begin{aligned}
& \underset{\bA}{\text{minimize}}
& & \EE [ \II \{ \bx\TT \bA\TT \bSig^{-1}_A \bdel_A > 0\} \neq y] \\
& \text{subject to} & & \bA \in \Real^{d \times p}.  
\end{aligned}
\end{equation}

When $(\bSig_A, \bdel_A)$ are not known, however, the optimization problem becomes non-convex. With $\bSig_A$ and $\bdel_A$ as above:
\begin{equation}
\begin{aligned}\label{eq:A}
& \underset{\bA, \bSig, \bdel}{\text{minimize}}
& & \EE [ \II \{ \bx\TT \bA\TT \bSig^{-1}_A \bdel_A > 0\} \neq y] \\
& \text{subject to} & & \bA \in \Real^{d \times p}.  
\end{aligned}
\end{equation}
While there are numerous approaches to solve related convex optimization problems through various sets of assumptions \cite{Fan2012a,Mai2012Feb}, we do not consider such techniques in this manuscript theoretically. This is because assuming either a structure for $\bSig_A$ or $\bdel_A$ presupposes an understanding of the properties of the feature space for wide data, which is often unsuitable if the dataset is large or has considerable complexity.

Let $\mc{A} = \{\bA : \bA \in \Real^{k \times p}, k \leq p\}$, $\mc{A}^d=\{\bA : \bA \in \Real^{d \times p}
\}$, and let $\mc{A}_* \subset \mc{A}$ be the set of $\bA$  that minimizes Eq.~\eqref{eq:A},
and let $\bA_* \in \mc{A}_*$.   Let $L_{\bA}^*=L_{\bA_*}$ be the misclassification rate for any $\bA \in \mc{A}_*$, that is, $L_{\bA}^*$ is the Bayes optimal misclassification rate for the classifier that composes $\bA$ with \Lda.}

In our opinion, Eq.~\eqref{eq:A} is the simplest supervised manifold learning problem there is: a two-class classification problem, where the data are multivariate Gaussians with shared, unknown covariances, the manifold is linear, and the classification is done via \Lda.
Nonetheless, solving Eq.~\eqref{eq:A} is difficult, because we do not know how to evaluate the integral analytically, and we do not know any algorithms that are guaranteed to find the global optimum in finite time. 
We proceed by studying a few natural choices for $\bA$.

\subsection{Bayes Optimal Projection}

\begin{lem}
$\bdel\TT  \bSig^{-1} \in \mc{A}_*$
\end{lem}

\begin{proof}
Let $\bB = (\bSig^{-1} \bdel)\TT = \bdel\TT (\bSig^{-1})\TT = \bdel\TT \bSig^{-1}$, so that $\bB\TT = \bSig^{-1} \bdel$,
and plugging this in to Eq.~\eqref{eq:g_A}. {\color{black}By the above, and noting the symmetry and invertibility of $\bSig$:
\begin{align*}
    \bSig_B &= \bB \bSig\bB\TT = \bdel\TT \bSig^{-1}\bSig(\bdel\TT \bSig^{-1})\TT \\
    &= \bdel\TT \bSig^{-1}\bSig\bSig^{-1}\bdel = \bdel\TT \bSig^{-1}\bdel \\
    \Rightarrow \bSig_B^{-1} &= \bdel^{-1} \bSig\bdel^{\TT{-1}} \\
    \pmb \delta_B &= \bB \bdel = \bdel\TT \bSig^{-1}\bdel
\end{align*}
We obtain:
\begin{align*}
g_{B}(x) &= \II \{ \bx\TT \bB\TT  \bSig^{-1}_{B} \bdel_{B} > 0\} &
\\&= \II \{ \bx\TT (\bSig^{-1} \bdel) (\bSig^{-1}_{B} \bdel_{B}) > 0\} & \text{plugging in $\bB$} \\
&=\II\{ \bx\TT (\bSig^{-1} \bdel) (\bdel^{-1} \bSig\bdel^{\TT{-1}} \bdel\TT \bSig^{-1}\bdel) > 0\} & \text{plug in $\bSig_B, \bdel_B$ from above}
\\&= \II \{ \bx\TT \bSig^{-1} \bdel  > 0\}
\end{align*}}
In other words, letting $\bB$ be the Bayes optimal projection recovers the Bayes classifier, as it should.
Or, more formally, for any $F \in \mc{F}_{\Lda}$, $L_{\bdel\TT \bSig^{-1}} = L_*$.
\end{proof}

\subsection[PCA]{Principle Components Analysis (\Pca) Projection}
\label{sec:pca}

Principle Components Analysis (\Pca) finds the directions of maximal variance in a dataset.  \Pca~is closely related to eigendecompositions and singular value decompositions (\Svd).  In particular, the top left singular vector of a matrix $\bX \in \Real^{p \times n}$, whose columns are centered, is the eigenvector with the largest eigenvalue of the centered covariance matrix $\bX \bX\TT$.  \Svd~enables one to estimate this eigenvector without ever forming the outer product matrix, because \Svd~factorizes a matrix $\bX$ into $\bU \bS \bV\TT$, where  $\bU$ and $\bV$ are orthonormal  ${p \times n}$ matrices, and $\bS$ is a diagonal matrix, whose diagonal values are decreasing,
$s_1 \geq s_2 \geq \cdots > s_n$.  Defining $\bU =[\bu_1, \bu_2, \ldots, \bu_n]$, where each $\bu_i \in \Real^p$, then $\bu_i$ is the $i^{th}$ eigenvector, and $s_i$ is the square root of the $i^{th}$ eigenvalue of $\bX \bX\TT$.
Let $\bA^{\Pca}_d =[\bu_1, \ldots , \bu_d]$ be the truncated \Pca~orthonormal matrix, and let $\bI_{d \times p}$ denote a $d \times p$ dimensional identity matrix.

The \Pca~matrix is perhaps the most obvious choice of an orthonormal matrix for several reasons.  First, truncated \Pca~minimizes the squared error loss between the original data matrix and all possible rank d representations:
\begin{align*}
\argmin_{A \in \Real^{d \times p}} \norm{ \bX - \bA^T \bA }_F^2.
\end{align*}
Second, the ubiquity of \Pca~has led to a large number of highly optimized numerical libraries for computing \Pca~(for example, LAPACK \cite{Anderson1999a}).

In this supervised setting, we consider two different variants of \Pca, each based on centering the data differently.  For the first one, which we refer to as ``pooled \Pca'' (or just \Pca~for brevity), we center the data by subtracting the ``pooled mean'' from each sample, that is, we let $\mtb{x}_i=\mb{x}-\bmu$, where $\bmu=\EE[\mb{x}]$.
For the second, which we refer to as ``class conditional \Pca'', we center the data by subtracting the ``class-conditional mean'' from each sample, that is, we let $\mtb{x}_i = \mb{x} - \bmu_y$, where $\bmu_y=\EE[\mb{x} | Y=y]$.

Notationally, let $\bU_d=[\bu_1,\ldots,\bu_d] \in \Real^{p \times d}$, and note that $\bU_d\TT \bU_d = \bI_{d \times d}$ and $\bU_d \bU_d\TT  = \bI_{p \times p}$.
Similarly, let $\bU \bS \bU\TT = \bSig$,
and $\bU \bS^{-1} \bU\TT = \bSig^{-1}$.  Let $\bS_d$ be the matrix whose diagonal entries are the eigenvalues, up to the $d^{th}$ one, that is $\bS_d(i,j)=s_i$ for $i=j \leq d$ and zero otherwise.  Similarly, $\bSig_d=\bU \bS_d \bU\TT=\bU_d \bS_d \bU_d\TT$.
Reduced-rank \Lda~({\color{black}\Rrlda}) is a regularized \Lda~algorithm.  Specifically, rather than using the full rank covariance matrix, it uses a rank-d approximation.  Formally,
let $g_{\Lda} := \II \{ x \bSig^{-1} \bdel > 0\}$ be the  \Lda~classifier,
and let {\color{black}$g_{\Lda}^d := \II \{ x \bSig_d^{-1} \bdel > 0\}$} be the regularized \Lda~classifier, that is, the \Lda~classifier where the the bottom $p-d$ eigenvalues of the covariance matrix are set to zero.


\begin{lem}
Using class-conditional \Pca~to pre-process the data, then using \Lda~on the projected data, is equivalent to {\color{black}\Rrlda}. 
\end{lem}

\begin{proof}
Plugging $\bU_d$ into Eq.~\eqref{eq:g_A} for $\bA$, and considering only the left side of the operand, we have
\begin{align*}
(\bA \bx)\TT \bSig^{-1}_A \bdel_A &= \bx\TT \bA\TT \bA \bSig^{-1} \bA\TT \bA \bdel,
\\&= \bx\TT  \bU_d\bU_d\TT \bSig^{-1} \bU_d\bU_d\TT \bdel,
\\&= \bx\TT  \bU_d \bU_d\TT \bU \bS^{-1} \bU\TT \bU_d\bU_d\TT \bdel,
\\&= \bx\TT  \bU_d \bI_{d \times p} \bS^{-1} \bI_{p \times d} \bU_d\TT \bdel,
\\&= \bx\TT  \bU_d \bS^{-1}_d  \bU_d\TT \bdel ,
\\&= \bx\TT  \bSig^{-1}_d  \bdel.
\end{align*}
\end{proof}

The implication of this lemma is that if one desires to implement {\color{black}\Rrlda}, rather than first learning the eigenvectors and then learning \Lda, one can instead directly implement regularized \Lda~by setting the bottom $p-d$ eigenvalues to zero.  This latter approach removes the requirement to run \Svd~twice, therefore reducing the computational burden as well as the possibility of numerical instability issues. We therefore refer to the projection composed of $d$  eigenvectors of the class-conditionally centered covariance matrices, $\bA_{\Lda}^d$.

\subsection[LOL]{Linear Optimal Low-Rank (\Lol) Projection}

The basic idea of \Lol~is to use both $\bdel$ and the top $d$ eigenvectors of the class-conditionally centered covariance.
When there are only two classes, $\bdel=\bmu_0-\bmu_1$.  When there are $C>2$ classes, there are $\binom{C}{2}=\frac{C!}{2 (C-2)!}$ pairwise combinations, $\bdel_{ij}=\bmu_i - \bmu_j$ for all $i\neq j$.   However, since $\binom{C}{2}$ is nearly $C^2$, when $C$ is large, this would mean incorporating many mean difference vectors.  Note that $[\bdel_{1,2}, \bdel_{1,3},\ldots, \bdel_{C-1,C}]$
is in fact a rank $C-1$ matrix, because it is a linear function of the $C$ different means. Therefore, we only need $C-1$ differences to span the space of all pairwise differences.  To mitigate numerical instability issues, we adopt the following convention.  For each class, estimate the expected mean and the number of samples per class, $\bmu_c$ and $\pi_c$.  Sort the means in order of decreasing $\pi_c$, so that $\pi_{(1)} > \pi_{(2)} > \cdots > \pi_{(C)}$.  Then, subtract $\bmu_{(1)}$ from all other means: $\bdel_i = \bmu_{(1)}-\bmu_{(i)}$, for $i=2,\ldots, C$.  Finally, {\color{black}$\bdel=[\bdel_2,\ldots, \bdel_C]$.}

Given $\bdel$ and $\bA_{\Lda}^{d-1}$, to obtain \Lol~na\"ively, we could simply concatenate the two, {\color{black}$\bA_{\Lol, naive}^d=[\bdel,\bA_{\Lda}^{d-1}]$.}
Recall that eigenvectors are orthonormal.  To maintain orthonormality between the eigenvectors and vectors of $\bdel$, we could easily apply Gram-Schmidt,  {\color{black}$\bA_{\Lol, naive}^d=$} \sct{Orth}$([\bdel, \bA_{\Lda}^{d-1}])$.
In practice, this orthogonalization step does not matter much, so we ignore it hereafter.
To ensure that $\bdel$ and $\bSig$ are balanced appropriately, we normalize each vector in $\bdel$ to have norm unity.  Formally, let $\mt{\bdel}_j = \bdel_j / \norm{\bdel_j}$, where $\bdel_j$ is the $j^{th}$ difference of the mean vector
 and let  $\bA_{\Lol}^d=[\mt{\bdel}, \bA_{\Lda}^{d-(C-1)}]$.

When the distribution of the data is not provided, each of the above terms must be estimated from the data.  We use the maximum likelihood estimators for each, specifically:
\begin{align}
n_c &= \sum_{i = 1}^n \II\{y_i = c\},\\
\mh{\pi}_c &= \frac{n_c}{n}, \\
\mh{\bmu} &= \frac{1}{n} \sum_{i=1}^n \mb{x}_i, \\
\mh{\bmu}_c &= \frac{1}{n_c} \sum_{i=1}^n \mb{x}_i \II\{y_i=c\}.
\end{align}

For completeness, below we provide pseudocode for learning the sample version of \Lol. The population version does not require the estimation of the parameters.


\subsection{{\color{black}\Rrlda}~is rotationally invariant}
\label{app:rot}

For certain classification tasks, the observed dimensions (or features) have intrinsic value, e.g. when simple interpretability is desired.  However, in many other contexts, interpretability is less important \cite{Breiman2001b}.  When the exploitation task at hand is invariant to rotations, then we have no reason to restrict our search space to be sparse in the observed dimensions. For example, we can consider sparsity in the eigenvector basis.
Let  $\bW$ be a rotation matrix, that is $\bW \in \mc{W}=\{\bW : \bW\TT = \bW^{-1}$ and det$(\bW)=1\}$.
Moreover, let $\bW \circ F$ denote the distribution $F$ after transformation by an operator $\bW$.  For example, if $F=\mc{N}(\bmu,\bSig)$ then $\bW \circ F=\mc{N}(\bW  \bmu, \bW \bSig \bW\TT)$.

\begin{defi}
A rotationally invariant classifier has the following property:
$$L_g^F = L_g^{W \circ F}, \qquad F \in \mc{F} \text{ and } W \in \mc{W}.$$
In words, the Bayes risk of using classifier $g$ on distribution $F$ is unchanged if $F$ is first rotated. 
\end{defi}

Now, we can state the main lemma of this subsection:  {\color{black}\Rrlda}~is rotationally invariant.
\begin{lem} \label{l:rot}
$L_{\Lda}^F = L_{\Lda}^{W \circ F}$, for any $F \in \mc{F}$.
\end{lem}

\begin{proof}
{\color{black}\Rrlda}~is in fact simply thresholding $\bx\TT \bSig^{-1} \bdel$ whenever its value is larger than some constant.
Thus, we can demonstrate rotational invariance by demonstrating that $\bx\TT \bSig^{-1} \bdel$ is rotationally invariant.


\begin{align*}
(\bW \bx) \TT  (\bW \bSig \bW\TT )^{-1} \bW \bdel  
&= \bx\TT \bW\TT  (\bW \bU \bS \bU\TT \bW\TT)^{-1} \bW \bdel & \text{by substituting $\bU \bS \bU\TT$ for $\bSig$} \\
&= \bx\TT \bW\TT  (\mt{\bU} \bS \mt{\bU}\TT)^{-1} \bW \bdel & \text{by letting $\mt{\bU}=\bW \bU$} \\
&= \bx\TT \bW\TT  (\mt{\bU} \bS^{-1} \mt{\bU}\TT) \bW \bdel & \text{by the laws of matrix inverse} \\
&= \bx\TT \bW\TT  \bW \bU \bS^{-1}  \bU\TT \bW\TT \bW \bdel & \text{by un-substituting $\bW \bU=\mt{\bU}$} \\
&= \bx\TT  \bU \bS^{-1}  \bU\TT  \bdel  & \text{because $\bW\TT \bW = \bI$} \\
&= \bx\TT   \bSig^{-1} \bdel & \text{by un-substituting $\bU \bS^{-1} \bU\TT = \bSig$}
\end{align*}
\end{proof}

One implication of this lemma is that we can reparameterize without loss of generality.  Specifically, defining $\bW := \bU\TT$ yields a change of variables: $\bSig \mapsto \bS$ and $\bdel \mapsto \bU\TT \bdel := \bdel''$, where $\bS$ is a diagonal covariance matrix.  Moreover, let $\bd=(\sigma_1,\ldots, \sigma_D)\TT$ be the vector of eigenvalues, then
$\bS^{-1} {\bdel'}=\bd^{-1} \odot \mt{\bdel}$, where $\odot$ is the Hadamard (entrywise) product.  The \Lda~classifier may therefore be encoded by a unit vector, $\mt{\bd}:= \frac{1}{m} \bd^{-1} \odot \mt{\bdel'}$, and its magnitude, $m:=\norm{\bd^{-1} \odot \mt{\bdel}}$.
This will be useful later.

\subsection{Rotation of Projection Based Linear Classifiers}

By a similar argument as above, one can easily show that:

\begin{align*}
(\bA  \bW \bx) \TT  (\bA \bW  \bSig  \bW\TT \bA\TT)^{-1} \bA \bW \bdel
&= \bx\TT (\bW\TT \bA\TT) (\bA \bW) \bSig^{-1} (\bW\TT \bA\TT) (\bA \bW) \bdel \\
&= \bx\TT \bY\TT \bY \bSig^{-1} \bY\TT \bY \bdel \\
&= \bx\TT \bZ \bSig^{-1} \bZ\TT \bdel \\
&= \bx\TT (\bZ \bSig \bZ\TT)^{-1} \bdel = \bx\TT \mt{\bSig}_d^{-1} \bdel,
\end{align*}
where $\bY = \bA \bW \in \Real^{d \times p}$ so that $\bZ=\bY\TT \bY$ is a symmetric ${p \times p}$ matrix of rank $d$.  In other words, rotating and then projecting is equivalent to a change of basis.
The implications of the above is:
\begin{lem}
$g_A$ is rotationally invariant if and only if span($\bA$)=span($\bSig_d$).
In other words, {\color{black}\Rrlda}~is the only rotationally invariant projection.
\end{lem}

\subsection{Low-Rank Canonical Correlation Analysis}
\label{sec:cca}

We now contrast \Lol~and low-rank CCA. For discriminant analysis, low-rank CCA corresponds to finding the eigenvectors of $S_{X}^{\dagger} S_{B}$ where $$S_{X} = \sum_{i} (X_i - \bar{X}) (X_i - \bar{X})^{\top}; \quad \bar{X} = \sum_{i} X_i $$
is the sample covariance matrix of the $X_i$, $S_{X}^{\dagger}$ is the inverse of $S_X$ (or Moore-Penrose pseudo-inverse of $S_X$ if $S_X$ is not invertible), and
$$ S_{B} = \frac{n_0}{n} (\bar{X}_0 - \bar{X}) (\bar{X}_0 - \bar{X})^{\top} + \frac{n_1}{n} (\bar{X}_1 - \bar{X}) (\bar{X}_1 - \bar{X})^{\top}; \quad \bar{X}_{j} = \sum_{i \colon Y_i = j} X_i \,\, \text{for $j \in \{0,1\}$} $$
is the between class covariance matrix \cite{Shin11}. It is widely known (see section 11.5 of \cite{mardia}) that if $S_{X}$ is invertible then the above formulation reduces to that of Fisher \:L, namely that of finding $\hat{v}$ satisfying
\begin{gather*} \hat{v} = \argmax_{v \not = 0} \frac{v^{\top} S_B v}{v^{\top} S_W v} \\ S_W = \sum_{i \colon Y_i = 0} (X_i - \bar{X}_0) (X_i - \bar{X}_0)^{\top} + \sum_{i \colon Y_i = 1} (X_i - \bar{X}_1) (X_i - \bar{X}_1)^{\top};
\end{gather*}
where $S_W$ is the pooled
within-sample covariance matrix and $S_X = S_W + S_B$.
In the context of our current paper where $X$ is assumed to be high-dimensional, it is well-known that $S_X$ is not a good estimator of the population covariance matrix $\Sigma_X = \mathbb{E}[(X - \mu) (X - \mu)^{\top}]$ and thus computing $S_X^{-1}$ is suboptimal for subsequent inference unless some form of regularization is employed. Our consideration of low-rank linear transformations $AX$ provides one principled approach to regularizations of high-dimensional $S_X$. In contrasts, the above (unregularized) formulation of low-rank CCA frequently yields discrimination direction vectors corresponding to ``maximum data piling'' (MDP) directions \cite{Shin11,maximum1} in high-dimensional settings (and always yield maximum data piling directions when $p \ge n$). These MDP directions lead to {\em perfect} discrimination of the training data, but can suffer from poor generalization performance,
as the examples in \cite{Shin11,maximum1} indicate.

Na\"ively computing the low-rank \Cca~projection requires storing and inverting a $p \times p$ matrix.  However, we devised an implementation for low-rank \Cca~that does not require ever materializing this matrix. Modern eigensolvers computes eigenvalues by performing a sequence of matrix vector multiplication. For example, to compute eigenvalues of $S_{X}$, an eigensolver performs $S_{X} v$ multiple times until the algorithm converges. Assume that the number of iteration is $i$, the computation complexity of the eigensolver is $O(n \times p \times i)$. Performing pseudo-inverse of $S_{X}$ computes truncated SVD on $S_{X}$, resulting in $S_{X} v = \sum_{i} (X_i - \bar{X}) ((X_i - \bar{X})^{\top} v)$.
Here we never physically generate $S_X$. Instead, we always compute $v' = (X_i - \bar{X})^{\top} v$ and then $v'' = (X_i - \bar{X}) v'$ to compute $S_{X} v$.
Assume $k$ classes, $S_{X} v$ has the computation complexity of $O(n \times p \times k)$ and the space complexity of $O(n \times p \times k)$. $S_{X}$ can be decomposed into
$U \Sigma V$, where $U$ is a $n \times n$ matrix and $V$ is a $n \times p$ matrix.
$$S_{X}^{\dagger} S_{B} = U \Sigma^{-1} V (\frac{n_0}{n} (\bar{X}_0 - \bar{X}) (\bar{X}_0 - \bar{X})^{\top} + \frac{n_1}{n} (\bar{X}_1 - \bar{X}) (\bar{X}_1 - \bar{X})^{\top}).$$
Computing eigenvalues of $S_{X}^{\dagger} S_{B}$ requires
$$S_{X}^{\dagger} S_{B} v = U \Sigma^{-1} V (\frac{n_0}{n} (\bar{X}_0 - \bar{X}) ((\bar{X}_0 - \bar{X})^{\top} v) + \frac{n_1}{n} (\bar{X}_1 - \bar{X}) ((\bar{X}_1 - \bar{X})^{\top} v)).$$
Similar to $S_{X} v$, we never physically generate $S_{X}^{\dagger}$ or $S_{B}$. Instead, we always multiply the terms on the right with $v$ first, which results in
the computation complexity of $O(n \times p)$ and the space complexity of $O(n \times p)$. To our knowledge, this algorithm is novel, and the implementation is also of course novel.

\section{Simulations}
\label{sec:simulations}


Let $f_{x|y}$ denote the conditional distribution of $X$ given $Y$, and let $f_y$ denote the prior probability of $Y$.  For simplicity, assume that realizations of the random variable $X$ are $p$-dimensional vectors,  $x \in \Real^p$, and realizations of the random variable $Y$ are binary, $y \in \{0,1\}$.
For most simulation settings, each class is Gaussian:
$f_{x|y} = \mc{N}(\bmu_y,\mb{\Sigma_y})$, where $\bmu_y$ is the class-conditional mean and $\mb{\Sigma}_y$ is the class-conditional covariance.  Moreover, we assume $f_y$ is a Bernoulli distribution with probability $\pi$ that $y=1$, $f_y = \mc{B}(\pi)$.
We typically assume that both classes are equally likely, $\pi=0.5$, and the covariance matrices are the same, $\bSig_0=\bSig_1=\bSig$. Under such assumptions, we merely specify $\bth=\{\bmu_0,\bmu_1, \bSig\}$. We consider the following simulation settings:

\paragraph*{Stacked Cigars}
\begin{compactitem}
\item $\bmu_0=\mb{0}$,
\item $\bmu_1=(a, b, a, \ldots, a)$,
\item $\bSig$ is a diagonal matrix, with diagonal vector, $\bd=(1,b,1\ldots,1)$,
\end{compactitem}
where $a=0.15$ and $b=4$.

\paragraph*{Trunk}
\begin{compactitem}
\item $\bmu_0=b/\sqrt{(1, 3, 5, \ldots, 2p)}$,
\item $\bmu_1=-\bmu_0$,
\item $\bSig$ is a diagonal matrix, with diagonal vector, $\bd=100/\sqrt{(p, p-1, p-2, \ldots, 1)}$,
\end{compactitem}
where $b=4$.



\paragraph*{Rotated Trunk} Same as Trunk, but the data are randomly rotated, that is, we sample $\mb{Q}$ uniformly from the set of p-dimensional rotation matrices, and then set:
\begin{compactitem}
\item $\bmu_0 \leftarrow \mb{Q} \bmu_0$,
\item $\bmu_1 \leftarrow \mb{Q} \bmu_1$,
\item $\bSig \leftarrow \mb{Q} \bSig \mb{Q}\TT$.
\end{compactitem}

%
%
%
%

\paragraph*{3 Classes} Same as Trunk, but with a third mean equal to the zero vector, $\bmu_2=\mb{0}$.
\begin{compactitem}
\item $\bmu_0=b/\sqrt{(1, 3, 5, \ldots, 2p)}$,
\item $\bmu_1=-\bmu_0$,
\item $\bmu_2 = \mb{0}$,
\item $\bSig$ is a diagonal matrix, with diagonal vector, $\bd=100/\sqrt{(p, p-1, p-2, \ldots, 1)}$,
\end{compactitem}
where $b=4$.

\paragraph*{Robust}
An experiment in which outliers are present for estimation of the projection matrix, but removed for training and testing of the classifier. This is due to the strong amount of noise present in the robust experiment will lead to poor generalizability of the estimated \Lda~classifier. Parameters indexed by $i$ correspond to the generative model for the inliers, and those with $o$ correspond to the outliers.
\begin{compactitem}
\item $\bmu_0^{(i)}=b/\sqrt{(1, 3, 5, \hdots, p)}$ for the first $p/2$ dimensions and 0 otherwise,
\item $\bmu_1^{(i)}=-\bmu_0$,
\item $\bSig^{(i)}=b^3/\sqrt{(1, 2, \hdots, p)}$,
\item $\bmu^{(o)}=0$,
\item $\bSig^{(o)}=b^6/\sqrt{(1, 2, \hdots, p)}$,
\item $\pi^{(i)} = 0.7$,
\item $\pi^{(o)} = 0.3$,
\end{compactitem}
and outliers are randomly assigned class $0$ or class $1$ with equal probability.

\paragraph*{Cross}
An experiment in which the two classes have identical means but different covariance matrices, meaning the optimal discriminant boundary is quadratic.
\begin{compactitem}
\item $\bmu_0 = \bmu_1 = \mb{0}$,
\item $\Sigma_0$ is a diagonal matrix, with diagonal $\left(a, \hdots, a, b, \hdots, b\right)$ where the first $\frac{d}{3}$ elements are $a$, and the rest are $b$,
\item $\Sigma_1$ is a diagonal matrix, with diagonal $\left(b, \hdots, b, a, \hdots, a, b, \hdots, b\right)$ where the middle $\frac{d}{3}$ elements are $a$, and the others are $b$,
\end{compactitem}
and we let $a=1$, and $b=\frac{1}{4}$.

\paragraph{Hump-$K$} An experiment with $K$ classes, in which the class means display an alternating series of humps, and the class covariance is a scalar multiple of the identity.
\begin{compactitem}
\item $\pi_k = \frac{1}{K}$
\item $x_{l, k} = \lfloor-\frac{K}{2}\rfloor$ the left endpoint of the hump
\item $x_{r, k} = d - x_{l, K-k+1}$ the right endpoint of the hump
\item $x_{m, k} = \left[\frac{x_{l, k} + x_{r, k}}{2}\right]$ the midpoint of the hump
\item Let $a_k, b_k, c_k$ be the unique coefficients such that $c_k + b_kx + a_kx^2$ passes through $x_{l,k}$ at $y=0$, passes through $x_{r,k}$ at $y=b$, and passes through $x_{m,k}$ at $y=0$.
\item Let $\alpha_k = \begin{cases}1 &  k\textrm{ is odd}\\ -1 & k\textrm{ is even}\end{cases}$
\item for $j = 1, \hdots, d$, let $\mu_{k, j} = \begin{cases}
0 & j \not\in [x_{l, k}, x_{r,k}] \\
c_k + b_k j + a_k j^2 & j \in [x_{l,k}, x_{r,k}]
\end{cases}$
\item $\pmb \Sigma$ is a diagonal matrix, with diagonal vector $(\sigma, \hdots, \sigma)$.
\end{compactitem} 
where $b=4$ and $\sigma=\frac{100}{K}$.
\paragraph*{Computational Efficiency Experiments} These experiments used the Trunk setting, increasing the observed dimensionality.

\paragraph*{Hypothesis Testing Experiments} We considered two related joint distributions here.  The first joint (Diagonal) is described by:
\begin{compactitem}
\item $\bmu_0=\mb{0}$,
\item $\mt{\bmu}_1 \sim \mc{N}(\mb{0}, \mb{I})$, $\bmu_1 = \mt{\bmu}_1 / \norm{\mt{\bmu}_1}$,
\item $\bSig$ is the same Toeplitz matrix (where the top row is $\rho^{(0,1,2,\ldots,p-1)}$), and the matrix is rescaled to have a Frobenius norm of $50$.
\end{compactitem}
The second (Dense) is the same except that the eigenvectors are uniformly random sampled orthonormal matrices, rather than the identity matrix.

\paragraph*{Regression Experiments} In this experiment we used a distribution similar to the Toeplitz distribution as described above, but $y$ was a linear function of $x$, that is, $y=Ax$, where $x \sim \mc{N}(\mb{0},\Sigma)$, where $\Sigma$ is the above described Toeplitz matrix, and $A$ is a diagonal matrix whose first two diagonal elements are non-zero, and the rest are zero.

%
%






\section{Theorems and Proofs of Main Result}
\label{sec:main}

\subsection{Chernoff information}
We now introduce the notion of the Chernoff information, which serves as our surrogate measure for the Bayes error of any classification procedure given the {\em projected} data. Our discussion of the Chernoff information is under the context of decision rules for hypothesis testing, nevertheless, as evidenced by the fact that the \emph{maximum a posteriori} decision rule---equivalently the Bayes classifier---achieves the Chernoff information rate, this distinction between hypothesis testing and classification is mainly for ease of exposition.

Let $F_0$ and $F_1$ be two absolutely continuous multivariate distributions in $\Omega \subset \mathbb{R}^{d}$ with density functions $f_0$ and $f_1$, respectively. Suppose that $X_1, X_2, \dots, X_m$ are independent and identically distributed random variables, with $X_i$ distributed either $F_0$ or $F_1$. We are interested in testing the simple null hypothesis $\mathbb{H}_0 \colon F = F_0$ against the simple alternative hypothesis $\mathbb{H}_1 \colon F = F_1$. A test $T$ is a sequence of mapping $T_m \colon \Omega^{m} \mapsto \{0,1\}$
such that given $X_1 = x_1, X_2 =x_2, \dots, X_m = x_m$, the test rejects $\mathbb{H}_0$ in favor of $\mathbb{H}_1$ if $T_m(x_1, x_2, \dots, x_m) = 1$; similarly, the test decides $\mathbb{H}_1$ instead of  $\mathbb{H}_0$ if $T_m(x_1, x_2, \dots, x_m) = 0$.
The Neyman-Pearson lemma states that, given $X_1 = x_1, X_2 = x_2, \dots, X_m = x_m$ and a threshold $\eta_m \in \mathbb{R}$, the likelihood ratio test rejects $\mathbb{H}_0$ in favor of $\mathbb{H}_1$ whenever
$$ \Bigl(\sum_{i=1}^{m} \log{f_0(x_i)} - \sum_{i=1}^{m} \log{f_1(x_i)} \Bigr) \leq \eta_m. $$
Moreover, the likelihood ratio test is the most powerful test at significance level $\alpha_m = \alpha(\eta_m)$, i.e., the likelihood ratio test minimizes the type II error $\beta_m$ subject to the constraint that the type I error is at most $\alpha_m$.

Assume that $\pi \in (0,1)$ is a prior probability of $\mathbb{H}_0$ being true. Then, for a given $\alpha_m^{*} \in (0,1)$, let $\beta_m^{*} = \beta_m^{*}(\alpha_m^{*})$ be the type II error associated with the likelihood ratio test when the type I error is at most $\alpha_m^{*}$. The quantity $\inf_{\alpha_m^{*} \in (0,1)} \pi \alpha_m^{*} + (1 - \pi) \beta_m^{*}$
is then the Bayes risk in deciding between $\mathbb{H}_0$ and $\mathbb{H}_1$ given the $m$ independent random variables $X_1, X_2, \dots, X_m$. A classical result of Chernoff \cite{chernoff_1952} states that the Bayes risk is intrinsically linked to a quantity known as the {\em Chernoff information}. More specifically, let $C(F_0, F_1)$ be the quantity
\begin{equation}
\label{eq:chernoff-defn}
\begin{split} C(F_0, F_1) & = - \log \, \Bigl[\, \inf_{t \in (0,1)} \int_{\mathbb{R}^{d}} f_0^{t}(\mb{x}) f_1^{1-t}(\mb{x}) \mathrm{d}\mb{x} \Bigr] \\
&= \sup_{t \in (0,1)} \Bigl[ - \log \int_{\mathbb{R}^{d}} f_0^{t}(\mb{x}) f_1^{1-t}(\mb{x}) \mathrm{d}\mb{x} \Bigr]
\end{split}
\end{equation}
Then we have
\begin{equation}
\label{eq:chernoff-binary}
\begin{split}
\lim_{m \rightarrow \infty} \frac{1}{m} \inf_{\alpha_m^{*} \in (0,1)} \log( \pi \alpha_m^{*} + (1 - \pi) \beta_m^{*}) & = - \, C(F_0, F_1).
\end{split}
\end{equation}
Thus $C(F_0, F_1)$ is the {\em exponential} rate at which the Bayes error $\inf_{\alpha_m^{*} \in (0,1)} \pi \alpha_m^{*} + (1 - \pi) \beta_m^{*}$ decreases as $m \rightarrow \infty$; we also note that the $C(F_0, F_1)$ is independent of $\pi$. We also define, for a given $t \in (0,1)$ the Chernoff divergence $C_t(F_0, F_1) $ between $F_0$ and $F_1$ by
$$ C_{t}(F_0,F_1) = - \log \int_{\mathbb{R}^{d}} f_0^{t}(\mb{x}) f_1^{1-t}(\mb{x}) \mathrm{d}\mb{x}. $$
The Chernoff divergence is an example of a $f$-divergence as defined in \cite{Csizar}. When $t = 1/2$, $C_t(F_0,F_1)$ is the Bhattacharyya distance between $F_0$ and $F_1$.

The result of Eq.~\eqref{eq:chernoff-binary} can be extended to $K + 1 \geq 2$ hypothesis, with the exponential rate being the minimum of the Chernoff information between any pair of hypothesis. More specifically, let $F_0, F_1, \dots, F_{K}$ be distributions on $\mathbb{R}^{d}$ and let $X_1, X_2, \dots, X_m$ be independent and identically distributed random variables with distribution $F \in \{F_0, F_1, \dots, F_K\}$. Our inference task is in determining the distribution of the $X_i$ among the $K+1$ hypothesis $\mathbb{H}_0 \colon F = F_0, \dots, \mathbb{H}_{K} \colon F = F_K$.
Suppose also that hypothesis $\mathbb{H}_k$ has {\em a priori} probabibility $\pi_k$. For any decision rule $g$, the risk of $g$ is $r(g) = \sum_{k} \pi_k \sum_{l \not = k} \alpha_{lk}(g) $ where $\alpha_{lk}(g)$ is the probability of accepting hypothesis $\mathbb{H}_l$ when hypothesis $\mathbb{H}_k$ is true. Then we have \cite{leang-johnson}
\begin{equation}
\label{eq:chernoff-multiple}
\inf_{g} \lim_{m \rightarrow \infty}  \frac{r(g)}{m} = - \min_{k \not = l} C(F_k, F_l),
\end{equation}
where the infimum is over all decision rules $g$, i.e., for any $g$, $r(g)$ decreases to $0$ as $m \rightarrow \infty$ at a rate no faster than $\exp(- m \min_{k \not = l} C(F_k, F_l))$.

When the distributions $F_0$ and $F_1$ are multivariate normal, that is, $F_0 =  \mathcal{N}(\mu_0, \Sigma_0)$ and $F_1 = \mathcal{N}(\mu_1, \Sigma_1)$; then, denoting by $\Sigma_t = t \Sigma_0 + (1 - t) \Sigma_1$, we have
\begin{equation*}
C(F_0, F_1) = \sup_{t \in (0,1)} \Bigl(\frac{t(1 - t)}{2} (\mu_1 - \mu_2)^{\top}\Sigma_t^{-1}(\mu_1 - \mu_2) + \frac{1}{2} \log \frac{|\Sigma_t|}{|\Sigma_0|^{t} |\Sigma_1|^{1 - t}}  \Bigr).
\end{equation*}

\subsection{Projecting data and Chernoff information}
We now discuss how the Chernoff information characterizes the effect a linear transformation $A$ of the data has on classification accuracy.
We start with the following simple result whose proof follows directly from Eq.~\eqref{eq:chernoff-multiple}.
\begin{lem}
\label{lem:chernoff-1}
Let $F_0 = \mathcal{N}(\mu_0, \Sigma)$ and $F_1 \sim \mathcal{N}(\mu_1, \Sigma)$ be two multivariate normals with equal covariance matrices. For any linear transformation $A$, let $F_0^{(A)}$ and $F_1^{(A)}$ denote the distribution of $AX$ when $X \sim F_0$ and $X \sim F_1$, respectively. We then have
\begin{equation}
\label{eq:lem:chernoff-1}
\begin{split}
C(F_0^{(A)}, F_1^{(A)}) &= \frac{1}{8} (\mu_1 - \mu_0)^{\top} A^{\top} (A \Sigma A^{\top})^{-1} A (\mu_1 - \mu_0) \\ & = \frac{1}{8} (\mu_1 - \mu_0)^{\top} \Sigma^{-1/2} \Sigma^{1/2} A^{\top} (A \Sigma A^{\top})^{-1} A \Sigma^{1/2} \Sigma^{-1/2} (\mu_1 - \mu_0) \\
&= \frac{1}{8} \|P_{\Sigma^{1/2} A^{\top}} \Sigma^{-1/2} (\mu_1 - \mu_0) \|_{F}^{2}
\end{split}
\end{equation}
where $P_{Z} = Z(Z^{\top} Z)^{-1} Z^{\top}$ denotes the matrix corresponding to the orthogonal projection onto the columns of $Z$.
\end{lem}

Thus for a classification problem where $X | Y = 0$ and $X | Y = 1$ are distributed multivariate normals with mean $\mu_0$ and $\mu_1$ and the same covariance matrix $\Sigma$, Lemma~\ref{lem:chernoff-1} then states that for any two linear transformations $A$ and $B$, the transformed data $AX$ is to be preferred over the transformed data $BX$ if
$$ (\mu_1 - \mu_0)^{\top} A^{\top} (A \Sigma A^{\top})^{-1} A (\mu_1 - \mu_0) > (\mu_1 - \mu_0)^{\top} B^{\top} (B \Sigma B^{\top})^{-1} B (\mu_1 - \mu_0). $$
In particular, using Lemma~\ref{lem:chernoff-1}, we obtain the following result showing the dominance of \Lol~over reduced-rank \Lda~(or simply {\color{black}\Rrlda}~for brevity) when the class conditional distributions are multivariate normal with a common variance.

\begin{thm}
\label{thm:Chernoff1}
Let $F_0 = N(\mu_0, \Sigma)$ and $F_1 \sim N(\mu_1, \Sigma)$ be multivariate normal distributions in $\mathbb{R}^{p}$. Let $\lambda_1 \geq \lambda_2 \geq \dots \lambda_p$ be the eigenvalues of $\Sigma$ and $u_1, u_2, \dots, u_p$ the corresponding eigenvectors. For $d \leq p$, let $U_{d} = [u_1 \mid u_2 \mid \dots \mid u_{d}] \in \mathbb{R}^{p \times d}$
be the matrix whose columns are the eigenvectors $u_1, u_2, \dots, u_{d}$. Let $A = [\delta \mid U_{d-1}]$ and $B = U_{d}$ be the \Lol~and {\color{black}\Rrlda}~linear transformations into $\mathbb{R}^{d}$, respectively.
Then
\begin{equation}
\label{eq:chernoff_arbitrary1}
\begin{split}
C(F_0^{(A)}, F_1^{(A)}) - C(F_0^{(B)}, F_1^{(B)}) &= \frac{(\delta^{\top} (I - U_{d-1} U_{d-1}^{\top}) \delta)^{2}}{\delta^{\top} (\Sigma - \Sigma_{d-1}) \delta} - \delta^{\top} (\Sigma_{d}^{\dagger} - \Sigma_{d-1}^{\dagger}) \delta \\ &
\geq \frac{1}{\lambda_d} \delta^{\top} (I - U_{d-1}U_{d-1}^{\top}) \delta - \frac{1}{\lambda_d} \delta^{\top} (U_{d} U_{d}^{\top} - U_{d-1} U_{d-1}^{\top}) \delta \geq 0
\end{split}
\end{equation}
and the inequality is strict whenever $\delta^{\top} (I - U_d U_d^{\top}) \delta > 0$.
\end{thm}

\begin{proof}
We first note that
\begin{equation*}
A \Sigma A^{\top} = \bigl[ \delta | U_{d-1} \bigr]^{\top}\, \Sigma \, \bigl[ \delta | U_{d-1} \bigr] = \begin{bmatrix} \delta^{\top} \Sigma \delta & \delta^{\top} \Sigma U_{d-1} \\ U_{d-1}^{\top} \Sigma \delta & U_{d-1}^{\top} \Sigma U_{d-1} \end{bmatrix} = \begin{bmatrix} \delta^{\top} \Sigma \delta & \delta^{\top} \Sigma U_{d-1} \\ U_{d-1}^{\top} \Sigma \delta & \Lambda_{d-1} \end{bmatrix}
\end{equation*}
where $\Lambda_{d-1} = \mathrm{diag}(\lambda_1, \lambda_2, \dots, \lambda_{d-1})$ is the $(d-1) \times (d-1)$ diagonal matrix formed by the eigenvalues $\lambda_1, \lambda_2, \dots, \lambda_{d-1}$.
Therefore, letting $\gamma = \delta^{\top} \Sigma \delta - \delta^{\top} \Sigma U_{d-1} \Lambda_{d-1}^{-1} U_{d-1}^{\top} \Sigma \delta$, we have
\begin{equation*}
\begin{split}
(A \Sigma A^{\top})^{-1} &= \begin{bmatrix} \delta^{\top} \Sigma \delta & \delta^{\top} \Sigma U_{d-1} \\ U_{d-1}^{\top} \Sigma \delta & U_{d-1}^{\top} \Sigma U_{d-1} \end{bmatrix}^{-1} \\
&= \begin{bmatrix} \gamma^{-1} & - \delta^{\top} \Sigma U_{d-1} \Lambda_{d-1}^{-1} \gamma^{-1} \\
- \Lambda_{d-1}^{-1} U_{d-1}^{\top} \Sigma \delta \gamma^{-1} & \bigl(\Lambda_{d-1} - \frac{U_{d-1}^{\top} \Sigma \delta \delta^{\top} \Sigma U_{d-1}}{\delta^{\top} \Sigma \delta} \bigr)^{-1} \end{bmatrix}.
\end{split}
\end{equation*}
The Sherman-Morrison-Woodbury formula then implies
\begin{equation*}
\begin{split}
\Bigl(\Lambda_{d-1} - \frac{U_{d-1}^{\top} \Sigma \delta \delta^{\top} \Sigma U_{d-1}}{\delta^{\top} \Sigma \delta} \Bigr)^{-1} & = \Lambda_{d-1}^{-1} + \frac{\Lambda_{d-1}^{-1} U_{d-1}^{\top} \Sigma \delta \delta^{\top} \Sigma U_{d-1} \Lambda_{d-1}^{-1} /(\delta^{\top} \Sigma \delta)}{1 - \delta^{\top} \Sigma U_{d-1} \Lambda_{d-1}^{-1} U_{d-1}^{\top} \Sigma \delta/(\delta^{\top} \Sigma \delta)} \\
&= \Lambda_{d-1}^{-1} + \frac{\Lambda_{d-1}^{-1} U_{d-1}^{\top} \Sigma \delta \delta^{\top} \Sigma U_{d-1} \Lambda_{d-1}^{-1}}{\delta^{\top} \Sigma \delta - \delta^{\top} \Sigma U_{d-1} \Lambda_{d-1}^{-1} U_{d-1}^{\top} \Sigma \delta} \\
&= \Lambda_{d-1}^{-1} + \gamma^{-1} \Lambda_{d-1}^{-1} U_{d-1}^{\top} \Sigma \delta \delta^{\top} \Sigma U_{d-1} \Lambda_{d-1}^{-1}
\end{split}
\end{equation*}

We note that $\Sigma U_{d-1} = U_{d-1} \Lambda_{d-1}$ and $U_{d-1}^{\top} \Sigma = \Lambda_{d-1} U_{d-1}^{\top}$ and hence
\begin{equation*}
\begin{split}
 \gamma = \delta^{\top} \Sigma \delta - \delta^{\top} \Sigma U_{d-1} \Lambda_{d-1}^{-1} U_{d-1}^{\top} \Sigma \delta &= \delta^{\top} \Sigma \delta - \delta^{\top} U_{d-1} \Lambda_{d-1} \Lambda_{d-1}^{-1} \Lambda_{d-1} U_{d-1}^{\top} \delta \\ & =  \delta^{\top} \Sigma \delta - \delta^{\top} U_{d-1} \Lambda_{d-1} U_{d-1}^{\top} \delta = \delta^{\top} (\Sigma - \Sigma_{d-1}) \delta
 \end{split}
 \end{equation*}
 where $\Sigma_{d-1} = U_{d-1} \Lambda_{d-1} U_{d-1}^{\top}$ is the best rank $d-1$ approximation to $\Sigma$ with respect to any unitarily invariant norm. In addition,
 $$ \Lambda_{d-1}^{-1} U_{d-1}^{\top} \Sigma \delta \delta^{\top} \Sigma U_{d-1} \Lambda_{d-1}^{-1} = \Lambda_{d-1}^{-1} \Lambda_{d-1} U_{d-1}^{\top} \delta \delta^{\top} U_{d-1} \Lambda_{d-1} \Lambda_{d-1}^{-1} = U_{d-1}^{\top} \delta \delta^{\top} U_{d-1}.$$
 We thus have
 \begin{equation*}
 (A \Sigma A^{\top})^{-1} = \begin{bmatrix} \gamma^{-1} & - \delta^{\top} \Sigma U_{d-1} \Lambda_{d-1}^{-1} \gamma^{-1} \\
- \Lambda_{d-1}^{-1} U_{d-1}^{\top} \Sigma \delta \gamma^{-1} & \bigl(\Lambda_{d-1} - \frac{U_{d-1}^{\top} \Sigma \delta \delta^{\top} \Sigma U_{d-1}}{\delta^{\top} \Sigma \delta} \bigr)^{-1} \end{bmatrix} = \begin{bmatrix} \gamma^{-1} & - \gamma^{-1} \delta^{\top} U_{d-1} \\ - \gamma^{-1} U_{d-1}^{\top} \delta & \Lambda_{d-1}^{-1} + \gamma^{-1}  U_{d-1}^{\top} \delta \delta^{\top} U_{d-1} \end{bmatrix}.
 \end{equation*}
 Therefore,
 \begin{equation*}
 \begin{split}
 \delta^{\top} A^{\top} (A \Sigma A^{\top})^{-1} A \delta & = \delta^{\top} \bigl[ \delta \mid U_{d-1} \bigr] \begin{bmatrix} \gamma^{-1} & - \gamma^{-1} \delta^{\top} U_{d-1} \\ - \gamma^{-1} U_{d-1}^{\top} \delta & \Lambda_{d-1}^{-1} + \gamma^{-1}  U_{d-1}^{\top} \delta \delta^{\top} U_{d-1} \end{bmatrix} \bigl[ \delta | U_{d-1} \bigr]^{\top} \delta \\
 &= \bigl[\delta^{\top} \delta \mid \delta^{\top} U_{d-1} \bigr] \begin{bmatrix} \gamma^{-1} & - \gamma^{-1} \delta^{\top} U_{d-1} \\ - \gamma^{-1} U_{d-1}^{\top} \delta & \Lambda_{d-1}^{-1} + \gamma^{-1}  U_{d-1}^{\top} \delta \delta^{\top} U_{d-1} \end{bmatrix} \begin{bmatrix} \delta^{\top} \delta \\ U_{d-1}^{\top} \delta \end{bmatrix} \\
 &= \gamma^{-1} (\delta^{\top} \delta)^{2} - 2 \gamma^{-1} \delta^{\top} \delta \delta^{\top} U_{d-1} U_{d-1}^{\top} \delta + \delta^{\top} U_{d-1} (\Lambda_{d-1}^{-1} + \gamma^{-1} U_{d-1}^{\top} \delta \delta^{\top} U_{d-1}) U_{d-1}^{\top} \delta \\
 &= \gamma^{-1} (\delta^{\top} \delta - \delta^{\top} U_{d-1} U_{d-1}^{\top} \delta)^{2} + \delta^{\top} U_{d-1} \Lambda_{d-1}^{-1} U_{d-1}^{\top} \delta \\
 &= \gamma^{-1} (\delta^{\top} (I - U_{d-1} U_{d-1}^{\top}) \delta)^{2} + \delta^{\top} \Sigma_{d-1}^{\dagger} \delta
 \end{split}
 \end{equation*}
 where $\Sigma_{d-1}^{\dagger}$ is the Moore-Penrose pseudo-inverse of $\Sigma_{d-1}$. The \Lda~projection matrix into $\mathbb{R}^{d}$ is given by $B = U_{d}^{\top}$ and hence
 \begin{equation}
 \delta^{\top} B^{\top} (B \Sigma B^{\top})^{-1} B \delta = \delta^{\top} U_{d} \Lambda_d^{-1} U_d^{\top} \delta = \delta^{\top} \Sigma_{d}^{\dagger} \delta.
 \end{equation}
 We thus have
\begin{equation*}
\begin{split}
C(F_0^{(A)}, F_1^{(A)}) - C(F_0^{(B)}, F_1^{(B)}) &= \gamma^{-1} (\delta^{\top} (I - U_{d-1} U_{d-1}^{\top}) \delta)^{2} - \delta^{\top} (\Sigma_{d}^{\dagger} - \Sigma_{d-1}^{\dagger}) \delta \\ &= \frac{(\delta^{\top} (I - U_{d-1} U_{d-1}^{\top}) \delta)^{2}}{\delta^{\top} (\Sigma - \Sigma_{d-1}) \delta} - \delta^{\top} (\Sigma_{d}^{\dagger} - \Sigma_{d-1}^{\dagger}) \delta \\ &
\geq \frac{(\delta^{\top} (I - U_{d-1} U_{d-1}^{\top}) \delta)^{2}}{\lambda_d \delta^{\top} (I - U_{d-1}U_{d-1}^{\top}) \delta} - \frac{1}{\lambda_{d}} \delta^{\top} u_{d} u_{d}^{\top} \delta
\\ &= \frac{1}{\lambda_d} \delta^{\top} (I - U_{d-1}U_{d-1}) \delta - \frac{1}{\lambda_d} \delta^{\top} (U_{d} U_{d}^{\top} - U_{d-1} U_{d-1}^{\top}) \delta \geq 0
\end{split}
\end{equation*}
where we recall that $u_{d}$ is the $d$-th column of $U_{d}$. Thus $C(F_0^{(A)}, F_1^{(A)}) \geq C(F_0^{(B)}, F_1^{(B)})$ always, and the inequality is strict whenever $\delta^{\top} (I - U_{d} U_{d}^{\top}) \delta > 0$.
\end{proof}
\begin{remark}
Theorem~\ref{thm:Chernoff1} can be extended to the case wherein the linear transformations are $A = [\delta \mid U_{d-1}]$ and $B = U_d$, respectively, such that
$U_{d}$ is an arbitrary $p \times d$ matrix with $U_{d}^{\top} U_d = I$, and $U_{d-1}$ is the first $d-1$ columns of $U_d$.
A similar derivation to that in the proof of Theorem~\ref{thm:Chernoff1} then yields
 \begin{gather}
 C(F_0^{(A)}, F_1^{(A)}) = \frac{(\delta^{\top} \Sigma^{-1/2} (I - V_{d-1} V_{d-1}^{\top}) \Sigma^{1/2} \delta)^2}{\delta^{\top} \Sigma^{1/2} (I - V_{d-1} V_{d-1}^{\top}) \Sigma^{1/2} \delta} + \delta^{\top} \Sigma^{-1/2} V_{d-1} V_{d-1}^{\top} \Sigma^{-1/2} \delta \\
 C(F_0^{(B)}, F_1^{(B)}) = \delta^{\top} \Sigma^{-1/2} V_d V_d^{\top} \Sigma^{-1/2} \delta
 \end{gather}
 where $V_{d} V_{d-1}^{\top} = \Sigma^{1/2} U_d (U_d^{\top} \Sigma U_d)^{-1} U_d^{\top} \Sigma^{1/2}$ is the orthogonal projection onto the column space of $\Sigma^{1/2} U_d$. Hence $C(F_0^{(A)}, F_1^{(A)}) > C(F_0^{(B)}, F_1^{(B)})$ if and only if
 \begin{equation}
 \label{eq:general1}
 \frac{(\delta^{\top} \Sigma^{-1/2} (I - V_{d-1} V_{d-1}^{\top}) \Sigma^{1/2} \delta)^2}{\delta^{\top} \Sigma^{1/2} (I - V_{d-1} V_{d-1}^{\top}) \Sigma^{1/2} \delta} > \delta^{\top} \Sigma^{-1/2} (V_d V_d^{\top} - V_{d-1} V_{d-1}^{\top}) \Sigma^{-1/2} \delta.
 \end{equation}
We recover Eq.~\ref{eq:chernoff_arbitrary1}
by letting $U_d$ be the matrix whose columns are the eigenvectors corresponding to the $d$ largest eigenvalue of $\Sigma$.
\end{remark}

We next present a result relating the Chernoff information for \Lol~and {\color{black}\Rrlda}.
\begin{thm}
\label{thm:Chernoff2}
Assume the setting of Theorem~\ref{thm:Chernoff1}.
Let $C = \tilde{U}_d^{\top}$ where $\tilde{U}_d$ is the $p \times d$ matrix whose columns are the $d$ largest eigenvectors of the {\em pooled} covariance matrix $\tilde{\Sigma} = \mathbb{E}[(X - \tfrac{\mu_0 + \mu_1}{2})(X - \tfrac{\mu_0 + \mu_1}{2})^{\top}]$.
Then $C$ is the linear transformation for \Pca~and
\begin{equation}
\begin{split}
C(F_0^{(A)}, F_1^{(A)}) - C(F_0^{(C)}, F_1^{(C)})
&=
\frac{(\delta^{\top} (I - U_{d-1} U_{d-1}^{\top}) \delta)^{2} } {\delta^{\top} (\Sigma - \Sigma _{d-1})\delta} + \delta^{\top} \Sigma_{d-1}^{\dagger} \delta -
\delta^{\top} \tilde{\Sigma}_d^{\dagger} \delta - \frac{(\delta^{\top} \tilde{\Sigma}_d^{\dagger} \delta)^2}{4 - \delta^{\top} \tilde{\Sigma}_d^{\dagger} \delta}
\\ &=
\frac{(\delta^{\top} (I - U_{d-1} U_{d-1}^{\top}) \delta)^{2} } {\delta^{\top} (\Sigma - \Sigma _{d-1})\delta} + \delta^{\top} \Sigma_{d-1}^{\dagger} \delta -
\frac{4 \delta^{\top} \tilde{\Sigma}_d^{\dagger} \delta}{4 - \delta^{\top} \tilde{\Sigma}_d^{\dagger} \delta}.
\end{split}
\end{equation}
where $\tilde{\Sigma}_d = \tilde{U}_d \tilde{S}_d \tilde{U}_d^{\top}$ is the best rank $d$ approximation to $\tilde{\Sigma} = \Sigma + \tfrac{1}{4} \delta \delta^{\top}$.
\end{thm}
\begin{proof}
Assume, without loss of generality, that $\mu_1 = -\mu_0 = \mu$. We then have
$$\tilde{\Sigma} = \mathbb{E}[X X^{\top}] = \pi \Sigma + \pi \mu_0 \mu_0^{\top} + (1 - \pi) \Sigma + (1 - \pi) \mu_1 \mu_1^{\top} = \Sigma + \mu \mu^{\top} = \Sigma + \tfrac{1}{4} \delta \delta^{\top}. $$

Therefore
$$ (C \Sigma C^{\top})^{-1} = \bigl(\tilde{U}_{d}^{\top} \Sigma \tilde{U}_d \bigr)^{-1} = \bigl( \tilde{U}_d^{\top} (\tilde{\Sigma} - \tfrac{1}{4} \delta \delta^{\top}) \tilde{U}_d\bigr)^{-1} = \bigl(\tilde{S}_d - \tfrac{1}{4} \tilde{U}_d^{\top} \delta \delta^{\top} \tilde{U}_d \bigr)^{-1} = \tilde{S}_d^{-1} + \frac{\tilde{S}_d^{-1} \tilde{U}_d^{\top} \delta \delta^{\top} \tilde{U}_d \tilde{S}_d^{-1}}{4 - \delta^{\top} \tilde{U}_d \tilde{S}_d^{-1} \tilde{U}_d^{\top} \delta} $$
where $\tilde{S}_d$ is the diagonal matrix containing the $d$ largest eigenvalues of $\tilde{\Sigma}$. Hence
\begin{equation}
\begin{split}
C(F_0^{(C)}, F_1^{(C)}) = \delta^{\top} C^{\top} (C \Sigma C^{\top})^{-1} C \delta & = \delta^{\top} \tilde{U}_d \Bigl(\tilde{S}_d^{-1} + \frac{\tilde{S}_d^{-1} \tilde{U}_d^{\top} \delta \delta^{\top} \tilde{U}_d \tilde{S}_d^{-1}}{4 - \delta^{\top} \tilde{U}_d \tilde{S}_d^{-1} \tilde{U}_d^{\top} \delta} \Bigr) \tilde{U}_d^{\top} \delta \\ &= \delta^{\top} \tilde{U}_d \tilde{S}_d^{-1} \tilde{U}_d^{\top} \delta + \frac{(\delta^{\top} \tilde{U}_d \tilde{S}_d^{-1} \tilde{U}_d^{\top} \delta)^{2}}{4 - \delta^{\top} \tilde{U}_d \tilde{S}_d^{-1} \tilde{U}_d^{\top} \delta} \\
&= \delta^{\top} \tilde{\Sigma}_d^{\dagger} \delta + \frac{(\delta^{\top} \tilde{\Sigma}_d^{\dagger} \delta)^2}{4 - \delta^{\top} \tilde{\Sigma}_d^{\dagger} \delta} = \frac{4 \delta^{\top} \tilde{\Sigma}_d^{\dagger} \delta}{4 - \delta^{\top} \tilde{\Sigma}_d^{\dagger} \delta}.
\end{split}
\end{equation}
as desired.
\end{proof}

\begin{remark}
We recall that the \Lol~projection $A = [\delta \mid U_{d-1}]^{\top}$ yields
$$ C(F_0^{(A)}, F_1^{(A)}) = \frac{(\delta^{\top} (I - U_{d-1} U_{d-1}^{\top}) \delta)^{2}}{\delta^{\top} (\Sigma - \Sigma_{d-1}) \delta} + \delta^{\top} \Sigma_{d-1}^{\dagger} \delta. $$

To illustrate the difference between the \Lol~projection and that based on the eigenvectors of the {\em pooled} covariance matrix, consider the following simple example. Let $\Sigma = \mathrm{diag}(\lambda_1, \lambda_2, \dots, \lambda_p)$ be a diagonal matrix with $\lambda_1 \geq \lambda_2 \geq \dots \geq \lambda_p$. Also let $\delta = (0,0,\dots,0,s)$. Suppose furthermore that $\lambda_p + s^{2}/4 < \lambda_d$. Then we have $\tilde{\Sigma}_d = \mathrm{diag}(\lambda_1, \lambda_2, \dots, \lambda_d, 0, 0, \dots, 0)$. Thus $\tilde{\Sigma}_d^{\dagger} = \mathrm{diag}(1/\lambda_1, 1/\lambda_2, \dots, 1/\lambda_d, 0,0, \dots, 0)$ and $\delta^{\dagger} \tilde{\Sigma}_d^{\dagger} \delta = 0$. Therefore, $C(F_0^{(B)}, F_1^{(B)}) = 0$.

On the other hand, we have
$$ C(F_0^{(A)}, F_1^{(A)}) = \frac{(\delta^{\top} (I - U_{d-1} U_{d-1}^{\top}) \delta)^{2}}{\delta^{\top} (\Sigma - \Sigma_{d-1}) \delta} + \delta^{\top} \Sigma_{d-1}^{\dagger} \delta = \frac{s^4}{s^2 \lambda_p} + 0 = s^{2}/\lambda_p.$$
\end{remark}
A more general form of the previous observation is the following result which shows that \Lol~is preferable over \Pca~when the dimension $p$ is sufficiently large.
\begin{proposition}
Let $\Sigma$ be a $p \times p$ covariance matrix of the form
$$ \Sigma = \begin{bmatrix} \Sigma_{d} & 0 \\ 0 & \Sigma_{d}^{\perp} \end{bmatrix} $$ where $\Sigma_{d}$ is a $d \times d$ matrix.
Let $\lambda_1 \geq \lambda_2 \geq \dots \geq \lambda_p$ be the eigenvalues of $\Sigma$, with $\lambda_1, \lambda_2, \dots, \lambda_d$ being the eigenvalues of $\Sigma_{d}$. Suppose that the entries of $\delta$ are i.i.d. with the following properties.
\begin{enumerate}
\item $\delta_i \sim Y_i \ast N(\tau, \sigma^2)$ where $Y_1, Y_2, \dots, Y_p \overset{\mathrm{i.i.d.}}{\sim} \mathrm{Bernoulli}(1 - \epsilon)$.
\item $p(1 - \epsilon) \rightarrow \theta$ as $p \rightarrow \infty$ for some constant $\theta$.
\end{enumerate}
Then there exists a constant $C > 0$ such that if $\lambda_d - \lambda_{d+1} \geq C \theta \tau^2 \log p$, then, with probability at least $\epsilon^{d}$
$$ C(F_0^{(A)}, F_1^{(A)}) > C(F_0^{(B)}, F_1^{(B)}) = 0
$$
\end{proposition}
\begin{proof}
The above construction of $\Sigma$ and $\delta$ implies, with probability at least $\epsilon^{d}$, that  the covariance matrix for $\tilde{\Sigma}$ is of the form
$$ \tilde{\Sigma} = \begin{bmatrix} \Sigma_{d} & 0 \\ 0 & \Sigma_{d}^{\perp} + \tfrac{1}{4} (\tilde{\delta} \tilde{\delta}^{\top}) \end{bmatrix} $$
where $\tilde{\delta} \in \mathbb{R}^{p - d}$ is formed by excluding the first $d$ elements of $\delta$.
Now, if $\lambda_{d+1} + \tfrac{1}{4} \|\tilde{\delta}\|^{2} < \lambda_d$, then
the $d$ largest eigenvalues of $\tilde{\Sigma}$ are still $\lambda_1, \lambda_2, \dots, \lambda_d$, and thus the eigenvectors corresponding to the $d$ largest eigenvalues of $\tilde{\Sigma}$ are the same as those for the $d$ largest eigenvalues of $\Sigma$. That is to say, $$ \lambda_{d+1} + \tfrac{1}{4} \|\tilde{\delta}\|^{2} < \lambda_d \Longrightarrow \tilde{\Sigma}_{d}^{\dagger} = \Sigma_{d}^{\dagger} \Longrightarrow \delta^{\top} \tilde{\Sigma}_{d}^{\dagger} \delta = 0 \Longrightarrow C(F_0^{(B)}, F_1^{(B)}) = 0.$$

We now compute the probability that $\lambda_{d+1} + \tfrac{1}{4} \|\tilde{\delta}\|^{2} < \lambda_d$. Suppose for now that $\epsilon > 0$ is fixed and does not vary with $p$. We then have
$$ \frac{\sum_{i=d+1}^{p} \delta_i^{2} - (p - d) (1 - \epsilon) \tau^2}{\sqrt{(p - d)(2 (1 - \epsilon) (2 \tau^2 \sigma^2 + \sigma^4) + \epsilon (1 - \epsilon) (\tau^4 + 2 \tau^2 \sigma^2 + \sigma^4))}} \overset{\mathrm{d}}{\longrightarrow} N(0,1).$$

Thus, as $p \rightarrow \infty$, the probability that $\lambda_{d+1} + \tfrac{1}{4} \|\tilde{\delta}\|^{2} < \lambda_d$ converges to that of
$$\Phi\Bigl(\frac{4(\lambda_d - \lambda_{d+1}) - (p - d) (1 - \epsilon) \tau^2}{\sqrt{(p - d)(2 (1 - \epsilon) (2 \tau^2 \sigma^2 + \sigma^4) + \epsilon (1 - \epsilon) (\tau^4 + 2 \tau^2 \sigma^2 + \sigma^4))}} \Bigr).$$ This probability can be made arbitrarily close to $1$ provided that $\lambda_{d} - \lambda_{d+1} \geq Cp(1 - \epsilon) \tau^2$ for all sufficiently large $p$ and for some constant $C > 1/4$. Since the probability that $\delta_1 = \delta_2 = \dots = \delta_d$ is at least $\epsilon^{d}$, we thus conclude that for sufficiently large $p$, with probability at least $\epsilon^{d}$,
$$C(F_0^{(B)}, F_1^{(B)}) = 0 < C(F_0^{(A)}, F_1^{(A)}).$$

In the case where $\epsilon = \epsilon(p) \rightarrow 1$ as $p \rightarrow \infty$ such that $p(1 - \epsilon) \rightarrow \theta$ for some constant $\theta$, then the probability that $\lambda_{d+1} + \tfrac{1}{4} \|\tilde{\delta}\|^{2} < \lambda_d$ converges to the probability that
$$\frac{1}{4} \sum_{i=1}^{K} \sigma^2 \chi_{1}^{2}(\tau) \geq \lambda_{d} - \lambda_{d+1} $$
where $K$ is Poisson distributed with mean $\theta$ and $\chi_{i}^{2}(\tau)$ is the non-central chi-square distribution with one degree of freedom and non-centrality parameter $\tau$. Thus if $\lambda_{d} - \lambda_{d+1} \geq C \theta \tau^2 \log{p}$ for sufficiently large $p$ and for some constant $C$, then this probability can also be made arbitrarily close to $1$.
\end{proof}

{\color{black}
\begin{remark}
The previous comparisons are done for the case of $C = 2$ classes. Extending these comparisons to the case of $C > 2$ classes is, however, non-trivial. More precisely, suppose we have $Y \in \{1,2,\dots, C\}$ and that, conditional on $Y = c$, $X \sim \mathcal{N}(\mu_c, \Sigma)$ is multivariate normal with mean $\mu_c$ and {\em common} covariance matrix $\Sigma$. Then, given $X = x$, the Bayes optimal classifier for $Y$ is still
$$g_{\mathrm{LDA}}(x) = \argmin_{y \in \{1,2,\dots,C\}} \Bigl[\frac{1}{2}(x - \mu_y)^{\top} \Sigma^{-1}(x - \mu_y) - \log \pi_y\Bigr] = \argmin_{y \in \{1,2,\dots,C\}} \Bigl[-x^{\top}\Sigma^{-1} \mu_y + \frac{1}{2} \mu_y^{\top} \Sigma^{-1} \mu_y - \log \pi_y
\Bigr]$$
Taking $\tfrac{1}{2} \mu_y^{\top} \Sigma^{-1} \mu_y - \log \pi_y$ as either a given constant or as an intercept term to be learned or estimated, the reduced-rank LDA for $C > 2$ classes still corresponds to looking at the top $d$ eigenvectors of $\Sigma$. That is to say, we transform the predictor variables via $x \mapsto U_d x$ followed by performing LDA on the transformed data. Similarly, the PCA transformation corresponds to using the top $d$ eigenvectors of the pooled covariance matrix $\tilde{\Sigma} = \mathbb{E}[(X - \sum_{c} \pi_c \mu_c)(X - \sum_{c} \pi_c \mu_c)^{\top}]$ followed by performing LDA.
Suppose we now compare LOL, rrLDA, and PCA in this multi-class setting. 
Let $A \colon X \mapsto AX$ be a linear transformation. Then by Eq.~\eqref{eq:chernoff-multiple} and Eq.~\eqref{eq:lem:chernoff-1},  the Chernoff information for the transformed data in this multi-class setting is 
\begin{equation*}
\begin{split}
\min_{c \not = c'} \frac{1}{8} (\mu_c - \mu_{c'})^{\top} A^{\top} (A \Sigma A^{\top})^{-1} A (\mu_c - \mu_{c'}). 
\end{split}
\end{equation*}
We now see that, in the case of rrLDA and PCA the linear transformation $A$ depends only on the covariance matrix $\Sigma$ and $\tilde{\Sigma}$, respectively. That is to say, the linear transformation $A$ does not depend on the choice of $c$ and $c'$. 
In contrast, currently for LOL the linear transformation $A$ depends on both $\Sigma$ as well as $\mu_c - \mu_{c'}$. In other words, there is no single choice for $A$ but rather that $A$ changes as $c, c'$ changes. Direct comparison, in the multi-classes setting, between LOL and either of rrLDA or PCA  is thus an open problem that we leave for future work.
Finally we note that if we allow the linear transformation for LOL to vary with the classes $c$ and $c'$, i.e., taking a one-vs-one approach to multi-classes classification, then the resulted presented in this paper are valid for all pairs $c, c'$.
\end{remark}}

\subsection{Finite Sample Performance}
We now consider the finite sample performance of \Lol~and \Pca-based classifiers in the high-dimensional setting with small or moderate sample sizes, e.g., when $p$ is comparable to $n$ or when $n \ll p$. Once again we assume that $X | Y = i \sim \mathcal{N}(\mu_i, \Sigma)$ for $i = 0, 1$. Furthermore, we also assume that $\Sigma$ belongs to the class $\Theta(p,r,k,\tau,\lambda)$ as defined below.

{\bf Definition} Let $\lambda > 0$, $\tau \geq 1$ and $k \leq p$ be given. Denote by $\Theta(p,r,k,\tau,\lambda,\sigma^2)$ the collection of matrices $\Sigma$ such that

$$ \Sigma = V \Lambda V^{\top} + \sigma^2 I $$
where $V$ is a $p \times r$  matrix with orthonormal columns and $\Lambda$ is a $r \times r$ diagonal matrix whose diagonal entries $\lambda_1, \lambda_2, \dots, \lambda_r$ satisfy $\lambda \geq \lambda_1 \geq \lambda_2 \geq \dots \geq \lambda_r \geq \lambda/\tau$. In addition, assume also that $|\mathrm{supp}(V)| \leq k$ where $\mathrm{supp}(V)$ denote the non-zero rows of $V$, i.e., $\mathrm{supp}(V)$ is the subset of $\{1,2,\dots,p\}$ such that $V_j \not = 0$ if and only if $j \in \mathrm{supp}(V)$.

We note that in general $r \leq k \ll p$ and $\lambda/\tau \gg \sigma^2$. We then have the following result.

\begin{thm}[\cite{cai-pca-1}]
\label{thm:cai1}
Suppose there exist constants $M_0$ and $M_1$ such that $M_1 \log p \geq \log n \geq M_0 \log \lambda$. Then there exists a constant $c_0 = c_0(M_0, M_1)$ depending on $M_0$ and $M_1$ such that for all $n$ and $p$ for which
$$ \frac{\tau k}{n} \log \frac{e p}{k} \leq c_0, $$ there exists an estimate $\hat{V}$ of $V$ such that

\begin{equation}
\label{eq:cai1} \sup_{\Sigma \in \Theta(p,r,k,\tau,\lambda,\sigma^2)} \mathbb{E} \|\hat{V} \hat{V}^{\top} - V V^{\top} \|^{2} \leq \frac{C k (\sigma \lambda + \sigma^2)}{n \lambda^{2}} \log \frac{e p}{k}
\end{equation}
where $C$ is a universal constant not depending on $p,r,k,\tau,\lambda$ and $\sigma^2$.
\end{thm}

Theorem~\ref{thm:cai1} then implies the following result for comparing the Chernoff information of the sample version of \Lol~against that for \Pca.

\begin{coro}
Let $\Sigma \in \Theta(p,r,k,\tau,\lambda)$ as defined above. Suppose that $C(F_0^{(A)}, F_1^{(A)}) > C(F_0^{(B)}, F_1^{(B)})$ where $A$ and $B$ denote the \Lol~and \Pca~projection matrices based on the eigenvectors of $\Sigma$ associated with the $d \leq r$ largest eigenvalues, i.e, $A = [\delta | V_{1:d-1}]$ and $B = V_{1:d}$. Then there exists constants $M$ and $c$ such that if
$\log n \geq M \log \lambda$ and $\tfrac{\tau k}{n} \log \tfrac{ep}{k} \leq c$,
then there exists an estimate $\hat{V}$ of $V$ such that, with $\hat{A} = [\hat{\delta} | \hat{V}_{1:d-1]}]$ and $\hat{B} = [\hat{V}_{1:d}]$, we have
$$ \mathbb{E}[C(F_0^{(\hat{A})}, F_1^{(\hat{A})})] > \mathbb{E}[C(F_0^{(\hat{B})}, F_1^{(\hat{B})})]$$
\end{coro}
The above corollary states that for $\Sigma \in \Theta(p,r,k,\tau,\lambda)$, then provided that the Chernoff information of the population version of \Lol~is larger than the Chernoff information of the population version of \Pca, we can choose $n$ sufficiently large (as compared to $\lambda$ and $\tau$ and $k$) such that the expected Chernoff information for the sample version of \Lol~is also larger than the expected Chernoff information of the sample version of \Pca. We emphasize that it is necessary that the \Lol~and the \Pca~version are both projected into the top $d \leq r$ dimension of the sample covariance matrices. The constants $M$ and $c$ in the statement of the above corollary are chosen so that $M$ (which depends on $M_0$ and $M_1$ in the statement of Theorem~\ref{thm:cai1}) is sufficiently large and $c$ (which depends on $c_0$) is sufficiently small to ensure that the bound in Eq.~\eqref{eq:cai1} is sufficiently small. If $C(F_0^{(A)}, F_1^{(A)}) > C(F_0^{(B)}, F_1^{(B)})$ and $\|\hat{V} \hat{V}^{\top} - VV^{\top}\|$ is sufficiently small, then $\mathbb{E}[C(F_0^{(\hat{A})}, F_1^{(\hat{A})})] > \mathbb{E}[C(F_0^{(\hat{B})}, F_1^{(\hat{B})})t]$ as desired.

\section{Real-Data Performance Analysis}
\label{app:neuroimaging}

\Pca, the industry-standard dimensionality reduction technique for high-dimensional problems, is compared to \Lol, \Rp, and {\color{black}\Rrlda}~ in terms of cross-validated classification error.

In the experiments, we used $k$ fold cross-validation. Testing sets were rotated across all folds, with the training sets comprising the remaining $k-1$ folds. A low-dimensional projection matrix $\pmb A$ is first learned through the training set, and the low-dimensional training points are then used to train an \Lda~classifier $C$. The testing points are embedded via $\pmb A$, and classification error is determined using the trained classifier $C$. 

Performance is assessed using Cohen's kappa \cite{cohen}, which normalizes the classification error typically between $0$ (the classifier performs no better than the classifier which guesses the most-likely class from the training set, the \texttt{chance}~ classifier) and $1$ (the trained classifier performs perfectly). Negative scores can be achieved if the trained classifier performs worse than the \texttt{chance}~ classifier. The effect size is measured as the difference between Cohen's kappa for the trained classifier after embedding with \Pca $\kappa(\Pca)$ and the trained classifier after embedding with technique $\varepsilon$, $\kappa(\varepsilon)$. Table \ref{tab:realdata} provides details about each neuroimaging dataset.

\begin{table}[h]
    \centering
    \begin{tabular}{| l | l | l | l | l | l | }
        \hline
        Problem & Sample Size ($n$) & Training Size$^*$ & $\#$ Features ($p$) & Classes ($K$) & Source \\
        \hline
        Templeton114 & 111  & 100 & $>1.5 \times 10^8$ & 2 & MRN \\
        BNU1 & 110 & 100 & $>1.5 \times 10^8$ & 2 & CoRR \cite{corr} \\
        BNU3 & 47 & 43 & $>1.5 \times 10^8$ & 2 & CoRR \cite{corr} \\
        SWU4 & 453 & 407 & $>1.5 \times 10^8$ & 2 & CoRR \cite{corr} \\
        KKI2009 & 42 & 38 & $>1.5 \times 10^8$ & 2 & KKI \cite{kki} \\
        Genomics & 340 & 306 & 745,184 & 2 & \citet{Douville2020Mar} \\
        \hline
    \end{tabular}
    \caption{Table of datasets used in this study. The top $5$ datasets (neuroimaging) are pre-processing by only registering the brains to the MNI152 template\cite{fsl1,fsl2,fsl3,mni152}. The neuroimaging dataset comprises a total of $5$ classification problems ($5$ datasets across a single sex classification task). The bottom dataset (genomics) is pre-processed by aligning sequencing data to $745,184$ amplicons on the human genome. The genomics dataset comprises two benchmark classification problems (sex or age).}
\label{tab:realdata}
\end{table}

\section{Extensions to Other Supervised Learning Problems}
\label{sec:other}

{\color{black} \subsection{Large numbers of classes}
Here, we explore an experiment in which the number of classes increases for a given simulation. We look at the multiclass hump-$K$ problem, described in Section \ref{sec:simulations}. In this simulation, while the space spanned by the differences of means conveys more information than the directions of maximal variance, we expect that the shift in the means for a given class at a given dimension should also increase the variance fractionally in that direction as well. Figure \ref{fig:khump}A shows the simulation setup, for $K=10$. Figure \ref{fig:khump}B indicates the misclassification rate as a function of Cohen's Kappa. We use Cohen's Kappa instead of the misclassification rate for direct evaluation since $K$ varies widely across these simulations at a fixed number of total samples $n=128$, making the difficulty of the problem as $K$ increases two-fold: not only are there more classes, but there are also fewer examples of each class per simulation setting. In all cases, the best random classifier would be the classifier that continually guesses a single class continuously, which has expected accuracy of $\frac{1}{K}$. On all simulations, we see that both \Pls\ and \Lol\ rapidly approach a higher Kappa statistic (better performance relative the random classifier) as they learn the space spanned by the differences of means. \Pls\ rapidly declines in performance as successive dimensions are added, and \Lol\ sees a small performance decline, as successive dimensions should convey no information regarding the class. \Pca\ is able to ultimately identify the space spanned by the differences of the means, but takes far more embedding dimensions to do so, and yields a lower Kappa statistic than either of the other two strategies.}

\begin{figure}
    \centering
    \includegraphics[width=\linewidth]{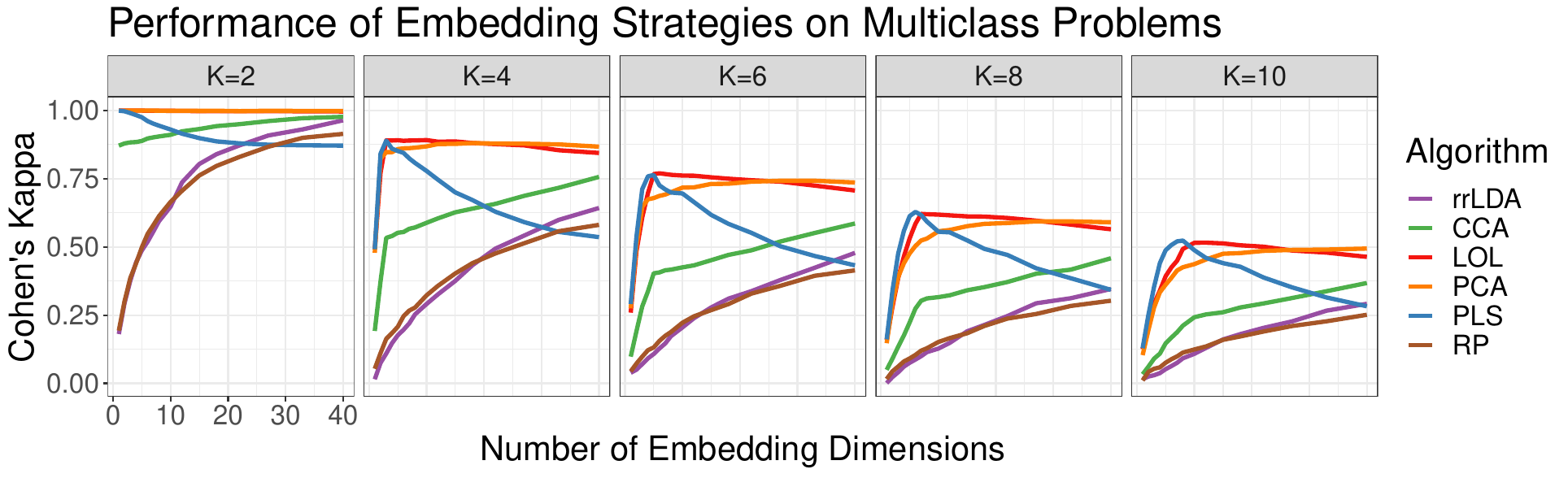}
    \caption{{\color{black}\textbf{The Multiclass Hump Simulation}. We show the results of the multiclass trunk problem, as the number of classes increases from $2$ to $10$, with the number of dimensions and the number of samples fixed. Effect size is measured with Cohen's Kappa. \Lol\ and \Pls\ provide better performance over competing techniques including \Pca, and this gap widens as the number of classes increases.}}
    \label{fig:khump}
\end{figure}

{\color{black}\subsection{Hypothesis Testing}}
The utility of incorporating the mean difference vector into supervised machine learning extends beyond  classification.  In particular, hypothesis testing can be considered as a special case of classification, with a particular loss function. We therefore apply the same idea to a hypothesis testing scenario.  The multivariate generalization of the t-test, called Hotelling's Test, suffers from the same problem as does the classification problem; namely, it requires inverting an estimate of the covariance matrix, which would result in a matrix that is low-rank and therefore singular in the high-dimensional setting.
To mitigate this issue in the hypothesis testing scenario, prior work applied similar tricks as they have done in the classification setting.
One particularly nice and related example is that of  Lopes et al. \cite{Lopes2011a}, who addresses this dilemma by using random projections to obtain a low-dimensional representation, following by applying Hotelling's Test in the lower-dimensional subspace.
Figure \ref{f:generalizations}{\color{magenta}A} and {\color{magenta}B} show the power of their test (labeled \sct{RP}) alongside the power of \Pca, \Lol,~and \Lfl~for two different conditions.
In each case we use the different approaches to project to low dimensions, followed by using Hotelling's test on the projected data.
In the first  example the true covariance matrix is diagonal, and in the second, the true covariance matrix is dense.
The horizontal axis on both panels characterizes the decay rate of the eigenvalues, so larger numbers imply the data is closer to low-rank
(see Methods for details).  The results indicate that the \Lol~test has higher power for essentially all scenarios.  Moreover, it is not merely replacing random projections with \Pca~(solid magenta line), nor simply incorporating the mean difference vector (dashed green line), but rather, it appears that \Lol~for testing uses both modifications to improve performance.

{\color{black}\subsection{Regression}}
High-dimensional  regression is another supervised learning method that can use the \Lol~idea. Linear regression, like classification and Hotelling's Test, requires inverting a  matrix as well.  By projecting the data onto a lower-dimensional subspace first, followed by linear regression on the low-dimensional data, we can mitigate the curse of high-dimensions.  To choose the projection matrix, we partition the data into $K$ partitions {\color{black}(we select $K=10$ arbitrarily)}, based on the percentile of the target variable, we obtain a K-class classification problem.  Then, we can apply \Lol~to learn the projection.  Figure \ref{f:generalizations}{\color{magenta}C} shows an example of this approach, contrasted with \Lasso~and partial least squares, in a sparse simulation setting (see Methods for details). \Lol~is able to find a better low-dimensional projection than \Lasso, and performs significantly better than partial least squares, for essentially all choices of number of dimensions to project into.

\begin{figure}
\centering
\includegraphics[width=1\linewidth]{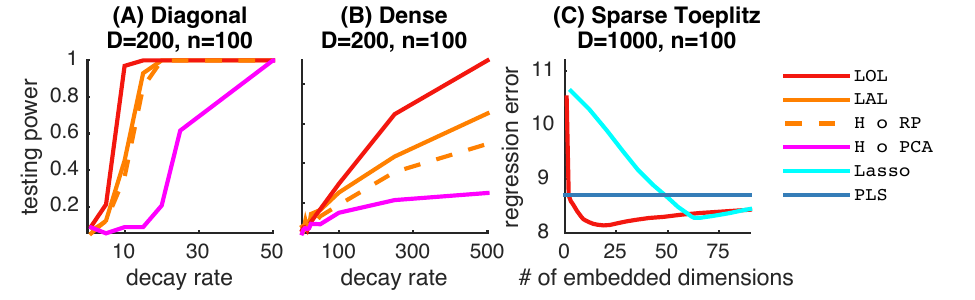}
\caption{
The intuition of including the mean difference vector is equally useful for other supervised manifold learning problems, including testing and regression.
\textbf{(A)} and \textbf{(B)} show two different high-dimensional testing settings, as described in Methods.  Power is plotted against the decay rate of the spectrum, which approximates the effective number of dimensions.  \Lol~composed with Hotelling's test outperforms the random projections variants described in \cite{Lopes2011a}, as well as several other variants.
\textbf{(C)} A sparse high-dimensional regression setting, as described in Methods, designed for sparse methods to perform well.  Log$_{10}$ mean squared error is plotted against the number of projected dimensions.
\Lol~composed with linear regression outperforms \sct{Lasso}~(cyan), the classic sparse regression method, as well as partial least squares (PLS; black).
These three simulation settings therefore demonstrate the generality of this technique.
}
\label{f:generalizations}
\end{figure}

\section{The R implementation of \Lol}

Figure \ref{Rimpl} shows the R implementation of \Lol~for binary classification
using FlashMatrix \cite{FlashMatrix}. The implementation takes a $D \times I$
matrix, where each column is a training instance and each instance has D
features, and outputs a $D \times k$ projection matrix.

\begin{figure}[h!]
\begin{lstlisting}[language=R]
LOL <- function(m, labels, k) {
	counts <- fm.table(labels)
	num.labels <- length(counts$val)
	num.features <- dim(m)[1]
	nv <- k - (num.labels - 1)
	gr.sum <- fm.groupby(m, 1, fm.as.factor(labels, 2), fm.bo.add)
	gr.mean <- fm.mapply.row(gr.sum, counts$Freq, fm.bo.div, FALSE)
	diff <- fm.get.cols(gr.mean, 1) - fm.get.cols(gr.mean, 2)
	svd <- fm.svd(m, nv=0, nu=nv)
	fm.cbind(diff, svd$u)
}
\end{lstlisting}
\caption{The R implementation of \Lol.}
\label{Rimpl}
\end{figure}

\begin{algorithm}[h!]
\caption{Simple pseudocode for two class \Lol~on sample data.}
\label{alg:LOL.train}
\begin{algorithmic}[1]
\Require $X$ a $p \times n$ matrix ($n \ll p$), where columns are observations; rows are
features.  An $n$ length vector of observation labels, $\mathbf y$.  
An integer $k$ to specify desired output dimension.
\Ensure $\mb{A} \in \Real^{p \times k}$
\Function{\texttt{LOL.train}}{$X$, $Y$, $k$}
\For{all $j \in J$}
\State ${n}_j = \sum_{i=1}^n \mb{I}(y_i = j) $ \Comment{sample size per class}
\State $\mh{\mu}_j = \frac{1}{n_j} \sum_{i=1}^n \mb{x}_i \mb{I}(y_i=j)$ \Comment{class means}
\EndFor
\State $\mh{\delta} = \mh{\mu}_1 - \mh{\mu}_2$ \Comment{difference of means}
\State $\mh{\delta} = \mh{\delta} / \norm{\mh{\delta}}$ \Comment{unit normalize difference of means}
\For{all $i \in [n]$}
\State $\mt{x}_i = \mb{x}_i - \mh{\mu}_{y_i}$ \Comment{class centered data}
\EndFor
\State $[\mh{u},\mh{d},\mh{v}]=$\texttt{svds}$(\mt{x},k-1)$ \Comment{compute top $k$ singular vectors}
\State $A=[\mh{\delta}, \mh{u}]$ \Comment{concatenate difference of the means and the top $k$ right singular vectors}



\EndFunction
\end{algorithmic}
\end{algorithm}




\end{document}